\documentclass{article}

\usepackage{microtype}
\usepackage{graphicx}
\usepackage{subfigure}
\usepackage{booktabs} 

\usepackage[toc,page,header]{appendix}
\usepackage{minitoc}

\usepackage{amsmath,amsthm,amssymb,amsfonts}
\usepackage{bbm}
\usepackage{mathtools}
\usepackage{extarrows}
\usepackage{comment}
\usepackage{algorithm}
\usepackage{algpseudocode,algorithmicx}

\newcommand{\RR}{\mathbb{R}}
\newcommand{\EE}{\mathbb{E}}
\newcommand{\PP}{\mathbb{P}}

\newcommand{\vtheta}{\boldsymbol{\theta}}

\DeclarePairedDelimiter{\dotp}{\langle}{\rangle}

\newtheorem{lemma}{Lemma}
\newtheorem{theorem}{Theorem}
\newtheorem{remark}{Remark}

\newtheorem{definition}{Definition}
\newtheorem{assumption}{Assumption}

\algrenewcommand{\algorithmiccomment}[1]{\hskip0em$\triangleright$ #1}
\algblock{ParFor}{EndParFor}
\algnewcommand\algorithmicparfor{\textbf{for}}
\algnewcommand\algorithmicpardo{\textbf{do in parallel}}
\algnewcommand\algorithmicendparfor{\textbf{end for}}
\algrenewtext{ParFor}[1]{\algorithmicparfor\ #1\ \algorithmicpardo}
\algrenewtext{EndParFor}{\algorithmicendparfor}

\usepackage{hyperref}

\usepackage[accepted]{icml2020}

\icmltitlerunning{VAFL: a Method of Vertical Asynchronous Federated Learning}

\begin{document}

\twocolumn[
\icmltitle{VAFL: a Method of Vertical Asynchronous Federated Learning}



\icmlsetsymbol{equal}{*}

\begin{icmlauthorlist}
\icmlauthor{Tianyi Chen}{rpi}
\icmlauthor{Xiao Jin}{rpi}
\icmlauthor{Yuejiao Sun}{ucla}
\icmlauthor{Wotao Yin}{ali}
\end{icmlauthorlist}

\icmlaffiliation{rpi}{Department of Electrical, Computer, and Systems Engineering, Rensselaer Polytechnic Institute, Troy, NY, USA.}
\icmlaffiliation{ucla}{Department of Mathematics,	University of California, Los Angeles, Los Angeles, CA, USA.}
\icmlaffiliation{ali}{DAMO Academy, Alibaba US, Seattle, WA, USA. Authors are listed in alphabetical order}

\icmlcorrespondingauthor{Wotao Yin}{wotao.yin@alibaba-inc.com}

\icmlkeywords{Machine Learning, ICML}

\vskip 0.3in
]



\printAffiliationsAndNotice{}  

\begin{abstract}
Horizontal Federated learning (FL) handles multi-client data that share the same set of features, and vertical FL trains a better predictor that combine all the features from different clients. 
This paper targets solving vertical FL in an asynchronous fashion, and develops a simple FL method. The new method allows each client to run stochastic gradient algorithms without coordination with other clients, so it is suitable for intermittent connectivity of clients. This method further uses a new technique of \emph{perturbed local embedding} to ensure data privacy and improve communication efficiency. Theoretically, we present the convergence rate and privacy level of our method for strongly convex, nonconvex and even nonsmooth objectives separately. Empirically, we apply our method to FL on various image and healthcare datasets. The results compare favorably to centralized and synchronous FL methods.

\end{abstract}

\doparttoc 
\faketableofcontents 

\vspace{-0.1cm}
\section{Introduction}

Federated learning (FL) is an emerging machine learning framework where a central server and multiple clients (e.g., financial organizations) collaboratively train a machine learning model \cite{Konecn2016Fed,mcmahan2017,bonawitz2017}. 
Compared with existing distributed learning paradigms, FL raises new challenges including 
the difficulty of synchronizing clients, the heterogeneity of data, and the privacy of both data and models. 

Most of existing FL methods consider the scenario where each client has data of a different set of  
subjects but their data share many common features. 
Therefore, they can collaboratively learn a joint mapping from the feature space to the label space. 
This setting is also referred to data-partitioned or horizontal FL \cite{konevcny2016federated,mcmahan2017}. 

Unlike the data-partitioned setting, in many learning scenarios, multiple clients handle data about the same set of subjects, but \emph{each client has a unique set of features}. This case arises in e-commerce, financial, and healthcare applications \cite{hardy2017private}. For example, an e-commerce company may want to predict a customer's credit using her/his historical transactions from multiple financial institutions; and, a healthcare company wants to evaluate the health condition of a particular patient using his/her clinical data from various hospitals \cite{sun2019privacy}.
In these examples, data owners (e.g., financial institutions and hospitals) have different records of those users in their joint user base, so by combining their features, they can establish a more accurate model. We refer to this setting as feature-partitioned or vertical FL \cite{yang2019fed}. 

Compared to the relatively well-studied horizontal FL setting \cite{mcmahan2017communication}, the vertical FL setting has its unique features and challenges \cite{hu2019fdml,kairouz2019advances}. 
In horizontal FL, the global model update at a server is an additive aggregation of the local models, which are updated by each client using its own data. In contrast, the global model in vertical FL is the concatenation of local models, which are coupled by the loss function, so updating a client's local model requires the information of the other clients' models. Stronger model dependence in the vertical setting leads to challenges on privacy protection and communication efficiency.

\subsection{Prior art}
We review prior work from the following three categories.

\noindent{\bf Federated learning.}
Since the seminal work \cite{Konecn2016Fed,mcmahan2017}, there has been a large
body of studies on FL in diverse settings. 
The most common FL setting is the horizontal setting, where a large set of data are partitioned among clients that share the same feature space \cite{konevcny2016federated}. 
To account for the personalization, multi-task FL has been studied in \cite{smith2017} that preserves the specialty of each client while also leveraging the similarity among clients, and horizontal FL with local representation learning has been empirically studied in \cite{liang2020think}. 
Agnostic FL has also been proposed in \cite{mohri2019icml}, where the federated model is optimized for any target distribution formed by a mixture of the client distributions. 
Communication efficiency has been an important issue in FL. 
Popular methods generally aim to: c1) reduce the number of bits per communication
round, including \citep{seide20141, alistarh2017qsgd,strom2015scalable, aji2017sparse}, to list a few; and, c2) save the number of communication
rounds \cite{chen2018lag,wang2018coop,liu2019communication,chen2020lasg}.

\noindent{\bf Privacy-preserving learning.}
More recently, feature-partitioned vertical FL has gained popularity in the financial and healthcare applications \cite{hardy2017private,yang2019fed,niu2019,kairouz2019}. 
Different from the aggregated gradients in the horizontal case, the local gradients in the vertical FL may involve raw data of those features owned by other clients, which raises additional concerns on privacy. 
Data privacy has been an important topic since decades ago \cite{yao1982}. 
But early approaches typically require expensive communication and signaling overhead when they are applied to the FL settings.
Recently, the notion of differential privacy becomes popular because i) it is a quantifiable measure of privacy \cite{dwork2014algorithmic,abadi2016,dong2019gaussian}; and, ii) many existing learning algorithms can achieve differential privacy via simple modifications. 
In the context of learning from multiple clients, it has been studied in \cite{bonawitz2017,hamm2016}. But all these approaches are not designed for the vertical FL models and the flexible client update protocols.

\noindent{\bf Asynchronous and parallel optimization.}
Regarding methodology, asynchronous and parallel optimization methods are often used to solve problems with asynchrony and delays, e.g., \cite{recht2011}. 
For the feature-partitioned vertical FL setting in this paper, it is particularly related to the Block Coordinate Descent (BCD) method \cite{xu2013,razaviyayn2013opt}. 
The asynchronous BCD and its stochastic variant have been developed under the condition of bounded delay in \cite{peng2016,lian2017nips,cannelli2016}. The Recent advances in this direction established convergence under unbounded delay with blockwise or stochastic update \cite{sun2017nips,dutta2018slow}. However, all the aforementioned works consider the shared memory structure so that each computing node can access the entire dataset, which significantly alleviates the negative effect of asynchrony and delays. Moreover, the state-of-the-art asynchronous methods cannot guarantee i) the convergence when the loss is nonsmooth, and, ii) the privacy of the local update which is at the epicenter of FL.

\subsection{This work}

The present paper puts forth an optimization method for vertical FL, which is featured by three main components. 

1. A \emph{general optimization formulation} for vertical FL that consists of a global model and one local embedding model for each client. 
The local embedding model can be linear or nonlinear, or even nonsmooth. It maps raw data to compact features and, thus, reduces the number of parameters that need to be communicated to and from the global model.

2. Flexible \emph{federated learning algorithms} that allow intermittent or even strategic client participation, uncoordinated training data selections, and data protection by differential-privacy based methods (for specific loss functions, one can instead apply multiple-party secure computing protocols).

3. \emph{Rigorous convergence analysis} that establishes the performance lower bound and the privacy level. 


We have also numerically validated our vertical FL algorithms and their analyses on federated logistic regression and deep learning. 
Tests on image and medical datasets demonstrate the competitive performance of our algorithms relative to centralized and synchronous FL algorithms. 

\section{Vertical federated learning}
This section introduces the formulation of vertical FL.

\subsection{Problem statement}
Consider a set of $M$ clients: ${\cal M}:=\{1,\ldots,M\}$. A dataset of $N$ samples, $\{\mathbf{x}_n, y_n\}_{n=1}^N$, are maintained by $M$ local clients. Each client $m$ is also associated with a unique set of features. For example, client $m$ maintains feature $x_{n,m}\in\mathbb{R}^{p_m}$ for $n=1,\ldots,N$, where $x_{n,m}$ is the $m$-th block of $n$-th sample vector $\mathbf{x}_n:=[x_{n,1}^{\top},\cdots,x_{n,M}^{\top}]^{\top}$ at client $m$. Suppose the $n$-th label $y_n$ is stored at the server. 

To preserve the privacy of data, the client data $x_{n,m}\in\mathbb{R}^{p_m}$ are not shared with other clients as well as the server. Instead, each client $m$ learns a local (linear or nonlinear) embedding $h_m$ parameterized by $\theta_m$ that maps the high-dimensional vector $x_{n,m}\in\mathbb{R}^{p_m}$ to a low-dimensional one $h_{n,m}:=h_m(\theta_m;x_{n,m})\in\mathbb{R}^{\underline{p_m}}$ with $\underline{p_m}\ll p_m$. 
Ideally, the clients and the server want to solve 
\begin{align}\label{eq.prob}
&F(\theta_0,\vtheta):=\frac{1}{N}\sum\limits_{n=1}^N\ell\left(\theta_0, h_{n,1},\ldots, h_{n,M}; y_n\right)+\sum_{m=1}^M r(\theta_m)\nonumber\\
& \text{with}~~~ h_{n, m}:=h_m(\theta_m;x_{n,m}),~~~m=1,\cdots,M
\end{align}
where $\theta_0$ is the global model parameter kept at and learned by the server, and $\vtheta:=[\theta_1^{\top},\cdots,\theta_M^{\top}]^{\top}$ concatenates the local models kept at and learned by local clients, $\ell$ is the loss capturing the accuracy of the global model parameters $\theta_0, \theta_1,\ldots, \theta_M$, and $r$ is the per-client regularizer that confines the complexity of or encodes the prior knowledge about the local model parameters.  

For problem \eqref{eq.prob}, the local information of client $m$ is fully captured in the embedding vector $h_{n,m},~\forall n=1,\cdots, N$. 
Hence, the quantities that will be exchanged between server and clients are $\{h_{n,m}\}$ and the gradients of $\ell(\theta_0, h_{n,1},\ldots, h_{n,M}; y_n)$ with respect to (w.r.t.) $\{h_{n,m}\}$. 
See a diagram for VAFL implementation in Figure \ref{fig:a-diag}.

\begin{algorithm}[t]
\caption{Vertical asynchronous federated learning}\label{alg:Async-SGD-FL}
    \begin{algorithmic}[1]
    \State{\textbf{initialize:}~$\theta_0$, $\{\theta_m\}$, datum index $n$, client index $m$}
    \While{not convergent}
        \State{\hspace{-7pt}{\bf when} a {\bf Client $m$} is activated, {\bf do}:}
           \State{$^\dag$select private datum (or data mini-batch) $x_{n,m}$}
            
            \State{\hspace{0pt}{\bf $^\dag$upload} secure information $h_{n,m}\!=\!h_m(\theta_m;x_{n,m})$}
            \State{\hspace{0pt}{\bf query} $\!\nabla_{h_{n,m}}\!\ell(\theta_0,h_{n,1},\ldots, h_{n,M}; y_n\!)$ from Server}
            \State{\hspace{0pt}update local model $\theta_m$}
        \vspace{3pt}
        \State{\hspace{-7pt}{\bf when Server} receives $h_{n,m}$ from Client $m$, {\bf do}:}
            \State{compute $\nabla_{\theta_0}\ell(\theta_0,h_{n,1},\ldots, h_{n,M}; y_n)$} 
            \State{update server's local model $\theta_0$}
        \vspace{3pt}
        \State{\hspace{-7pt}{\bf when Server} receives a query from Client $m$, {\bf do}:}
            \State{compute $\!\nabla_{h_{n,m}}\!\ell(\theta_0,h_{n,1},\ldots, h_{n,M}; y_n)$}
            \State{{\bf send} it to Client $m$}
    \EndWhile
    \end{algorithmic}
    $^\dag$We can let Step 5 also send  $h_{n,m}$ for those $n$ \emph{not} selected in Step 4. We can re-order Steps 4--7 as 6, 7, \emph{then} 4, and 5. They reduce information delay, yet analysis is unchanged.
\end{algorithm}

\begin{algorithm}[t]
\caption{Vertical $t$-synchronous federated learning}\label{alg:ksync-SGD-FL}
    \begin{algorithmic}[1]
    \State{\textbf{Initialize:}~$\theta_0$, $\{\theta_m\}$, datum index $n$, client index $m$, integer $1\leq t\leq M$}
    \While{not convergent}
        \Statex{\hspace{8pt}Algorithm \ref{alg:Async-SGD-FL}, Lines 3--7}
        \setcounter{ALG@line}{7}
        \State{\hspace{-7pt}{\bf when Server} receives $h_{n,m}$'s from $t$ Clients, {\bf do}:}
        \Statex{\hspace{15pt}Algorithm \ref{alg:Async-SGD-FL}, Lines 9 and 10}
        \setcounter{ALG@line}{10}
        \State{\hspace{-7pt}{\bf when Server} receives queries from $t$ Clients, {\bf do}:}
        \Statex{\hspace{15pt}Algorithm \ref{alg:Async-SGD-FL}, Lines 12 and 13 for each of $t$ clients}
        \setcounter{ALG@line}{13}
   \EndWhile
    \end{algorithmic}
\end{algorithm}

\begin{figure}[t]
\centering
\includegraphics[width=.45\textwidth]{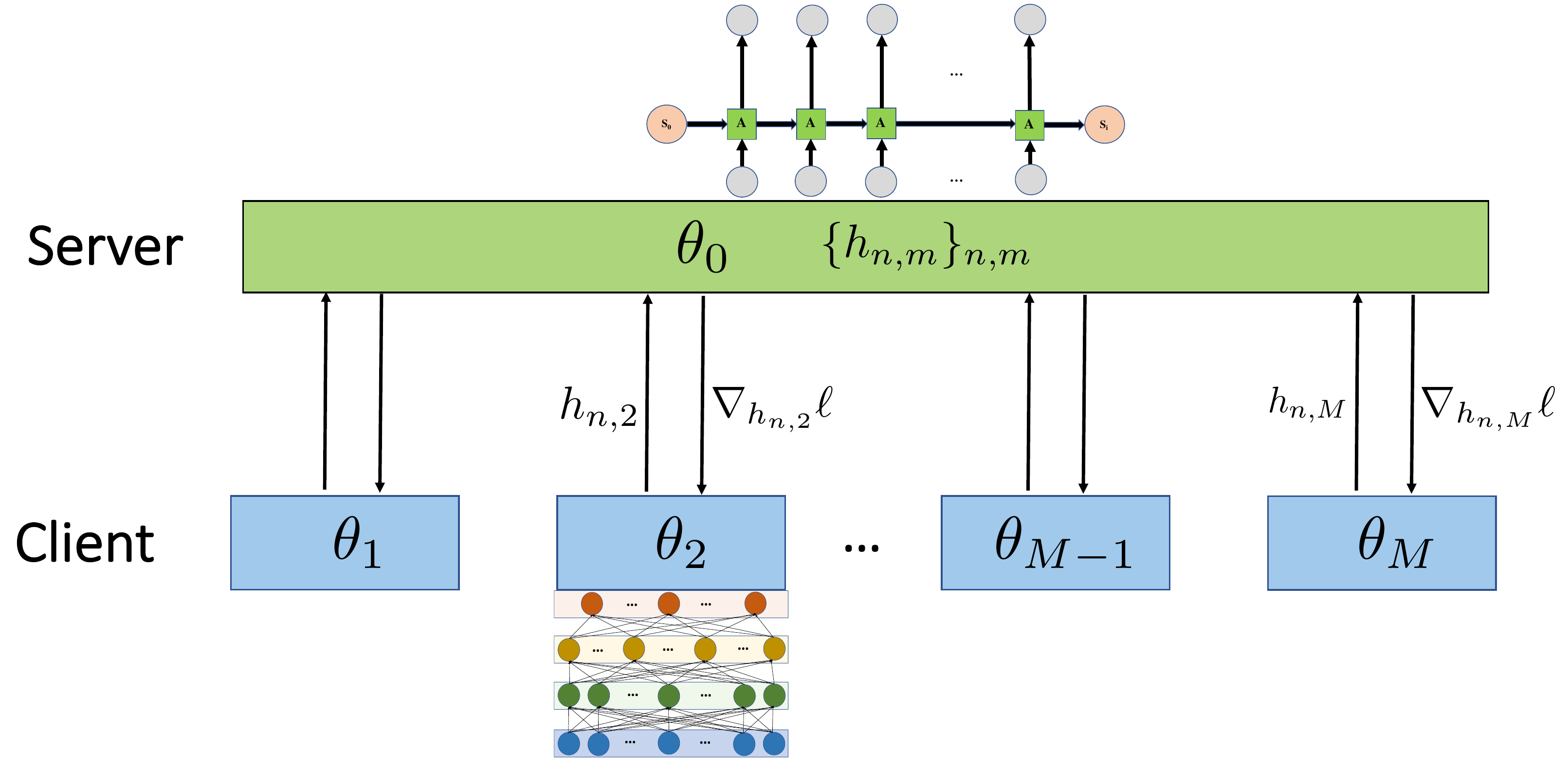}
\vspace{-0.2cm}
  \caption{A diagram for VAFL. The local model at client $m$ is denoted as $\theta_m$ which generates the local embedding $h_{n,m}$.}
\label{fig:a-diag}
\vspace{-0.4cm}
\end{figure}

\subsection{Asynchronous client updates}\label{subsec.avfl}

For FL, we consider solving \eqref{eq.prob} without coordination among clients. 
Asynchronous optimization methods have been used to solve such problems.
However, state-of-the-art methods cannot guarantee i) the convergence when the mapping $h_{n,m}$ is nonlinear (thus the loss is nonsmooth), and, ii) the privacy of the update which is at the epicenter of the FL paradigm. 

We first describe our vertical asynchronous federated learning (VAFL) algorithm in a high level as follows. 
During the learning process, from the server side, it waits until receiving a message from an active client $m$, which is either\\ 
i) \textbf{a query} of the loss function's gradient w.r.t. to the embedding vector $h_{n,m}$; or,\\ 
ii) \textbf{a new embedding vector} $h_{n,m}$ calculated using the updated local model parameter $\theta_m$. 

To response to the query i), the server calculates the gradient for client $m$ using its current $\{h_{n,m}\}$, and sends it to the client; and, upon receiving ii), the server computes the new gradient w.r.t. $\theta_0$ using  the embedding vectors it currently has from other clients and updates its model $\theta_0$. 

For each interaction with server, each \emph{active} client $m$ randomly selects a datum $x_{n,m}$, queries the corresponding gradient w.r.t. $h_{n,m}$ from server, and then it securely uploads the updated embedding vector $h_{n,m}$, and then updates the local model $\theta_m$. 
The mechanism that ensures secure uploading will be described in Section 4. 
Without introducing cumbersome iteration, client, and sample indexes, we summarize the asynchronous client updates in Algorithm \ref{alg:Async-SGD-FL}. 

Specifically, since clients update the model without external coordination, we thereafter use $k$ to denote the global counter (or iteration), which increases by one whenever i) the server receives the new embedding vector $h_{n,m}$ from a client, calculates the gradient, and updates the server model $\theta_0$; and, ii) the corresponding client $m$ obtains the gradient w.r.t. $h_{n,m}$, and updates the local model $\theta_m$. Accordingly, we let $m_k$ denote the client index that uploads at iteration $k$, and $n_k$ denote the sample index used at iteration $k$. 

For notation brevity, we use a single datum $n_k$ for each uncoordinated update in the subsequent algorithms, but the algorithm and its analysis can be easily generalized to a minibatch of data ${\cal N}_k$.
Let $\hat{g}_0^k$ denote the stochastic gradients of the loss at $n_k$-th sample w.r.t. server model $\theta_0$ as
\begin{subequations}\label{eqn:gradient}
\begin{align}
 \hat{g}_0^k:=\nabla_{\theta_0}\ell\big(\theta_0^k,h_{n_k,1}^{k-\tau_{n_k,1}^k},\ldots, h_{n_k,M}^{k-\tau_{n_k,M}^k}; y_{n_k}\big)    
\end{align}
and the gradients w.r.t. the local model $\theta_m$ as
\begin{align}
  \hat{g}_m^k\!:=&\nabla_{\theta_m}\ell\big(\theta_0^k,h_{n_k,1}^{k-\tau_{n_k,1}^k},\ldots, h_{n_k,M}^{k-\tau_{n_k,M}^k}; y_{n_k}\big)  \\
  =&\nabla_{\theta_m} h_{n_k,m}^k\! \nabla_{h_{n_k,m}}\ell \big(\theta_0^k,h_{n_k,1}^{k-\tau_{n_k,1}^k}\!\!\!,\ldots, h_{n_k,M}^{k-\tau_{n_k,M}^k}\!; y_{n_k} \big) .  \nonumber
\end{align}
\end{subequations}
The delay for client $m$ and sample $n$ will increase via
\begin{align} \label{eq.tau}
\tau_{n,m}^{k+1}  =
  \begin{cases}
  ~~~  1, \quad\quad &m= m^k,\, n=n^k, \\
  ~~~\tau_{n,m}^k+1, \quad\quad &{\rm otherwise}.
  \end{cases}
\end{align}

With the above short-hand notation, at iteration $k$, the update at the {\bf server side} is $\theta_0^{k+1}=\theta_0^k-\eta_0^k  \hat{g}_0^k$.
For the {\bf active local client $m_k$} at iteration $k$, its update is
\begin{subequations}
\begin{equation}    
    \theta_{m_k}^{k+1}=
        \theta_{m_k}^k-\eta_{m_k}^k \hat{g}_{m_k}^k-\eta_{m_k}^k \nabla r(\theta_{m_k}^k),
\end{equation}
and for the {\bf other clients $m\neq m_k$}, the update is
\begin{equation}    
    \theta_m^{k+1}=\theta_m^k,
\end{equation}
\end{subequations}
where $\eta_m^k$ is the stepsize and $m_k$ is the index of the client responsible for the $k$th update.

\subsection{Types of flexible update rules}

As shown in \eqref{eqn:gradient}, the stochastic gradients are evaluated using delayed local embedding information {\small$h_m^{k-\tau_{n_k,m}^k}$} from each client $m$, where $\tau_{n_k,m}^k$ is caused by both asynchronous communication and stochastic sampling. 

To ensure convergence, we consider two reasonable settings on the flexible update protocols:\\
1. \emph{Uniformly bounded delay $D$.} We can realize this by modifying the server behavior. 
During the training process, whenever the delay of $\tau_{n_k,m}^k$ exceeds $D(>0)$, the server immediately queries fresh $h_{n,m}$ from client $m$ before continuing the server update process.\\
2. \emph{Stochastic unbounded delay.} In this case, the activation of each client is a stochastic process. The delays is determined by the hitting times of the stochastic processes. For example, if the activation of all the clients follows independent Poisson processes, the delays will be geometrically distributed.\\
3. \emph{$t$-synchronous update, $t>0$.}
While fully asynchronous update is most flexible, $t$-synchronous update is also commonly adopted. In this case, the server computes the gradient w.r.t. $\theta_0$ until receiving $\{h_{n,m}\}$ from $t$ different clients, and then updates the server's model using the newly computed gradient. The $t$-synchronous updates have more stable performance empirically, which is listed in Algorithm \ref{alg:ksync-SGD-FL}. 


\section{Convergence analysis}\label{sec:convergence}
We present the convergence results of our VAFL method for the nonconvex and strongly convex cases and under different update rules. 
Due to space limitation, this section mainly presents the convergence rates for fully asynchronous version of VAFL (Algorithm \ref{alg:Async-SGD-FL}), and the convergence results for $t$-synchronous one (Algorithm \ref{alg:ksync-SGD-FL}) are similar, and thus will be given in the supplementary materials. 


To analyze the performance of Algorithm \ref{alg:Async-SGD-FL}, we first make the following assumptions on sampling and smoothness. 
\begin{assumption}\label{assump:sample}
Sample indexes $\{n_k\}$ are i.i.d. And the variance of gradient follows $    \EE\left[\|g_m^k-\nabla_{\theta_m}F(\theta_0^k,\vtheta^k]\|^2\right]\leq\sigma_m^2, ~ \forall m$, 
where $g_m^k$ is the stochastic gradient $\hat{g}_m^k$ without delay, e.g., $g_m^k:=\nabla_{\theta_m}\ell\left(\theta_0^k,h_{n_k,1}^k,\ldots, h_{n_k,M}^k; y_{n_k}\right)$.
\end{assumption}

\begin{assumption}\label{assump:lip}
The optimal loss is lower bounded $F^*>-\infty$. 
The gradient $\nabla F$ is $L$-Lipschitz continuous, and $\nabla_{\theta_m}F$ is $L_m$-Lipschitz continuous.
\end{assumption}

Generally, assumption \ref{assump:lip} cannot be satisfied under our general vertical FL formulation with \emph{nonsmooth} local embedding functions such as neural networks. However, techniques that we call perturbed local embedding will be introduced to enforce smoothness in Section \ref{sec:randomneuron}.

To handle asynchrony, we need the following assumption, which is often seen in the analysis of asynchronous BCD.

\begin{assumption}\label{assump:client}
The uploading client $m_k$ is independent of $m_0,\ldots, m_{k-1}$ and satisfies $\PP(m_k=m):= q_m$.
\end{assumption}
A simple scenario satisfying this assumptions is that the activation of all clients follows independent Poisson processes. That is, if the time difference between two consecutive activations of client $m$ follows exponential distribution with parameter $\lambda_m$, then the activation of client $m$ is a Possion process with $q_m=\lambda_m^{-1}/{\textstyle\sum_{j=1}^M}\lambda_j^{-1}$.

We first present the convergence results for bounded $\tau_{n_k,m}^k$.

\subsection{Convergence under bounded delay}\label{subsec.bounded}
We make the following assumption \emph{only} for this subsection.
\begin{assumption}[Uniformly bounded delay]\label{assump:boundeddelay}
For each client $m$ and each sample $n$, the delay $\tau_{n,m}^k$ at iteration $k$ is bounded by a constant $D$, i.e., $\tau_{n,m}^k\leq D$.
\end{assumption}


We first present the convergence for the nonconvex case.
\begin{theorem}\label{thm:nc1}
Under Assumptions \ref{assump:sample} -- \ref{assump:boundeddelay}, if $\eta_0^k=\eta_m^k=\min\{\frac{1}{4(1+D)L}, \frac{c_{\eta}}{\sqrt{K}}\}$ with $c_{\eta}\!>\!0$, then we have
\begin{equation}
\small
    \frac{1}{K}\sum\limits_{k=0}^{K-1}\EE[\|\nabla F(\theta_0^k,\vtheta^k)\|^2]={\cal O}\left({1}/{\sqrt{K}}\right).
\end{equation}
\end{theorem}

Under the additional assumption of strong convexity, the convergence rate is improved.
\begin{theorem}\label{thm:sc1}
Under Assumptions \ref{assump:sample} -- \ref{assump:boundeddelay}, and the additional assumption that $F$ is $\mu$-strongly convex in $(\theta_0,\vtheta)$, if $\eta^k=\frac{4}{\mu\min\limits_m\sqrt{q_m}(k+K_0)}$ with the constant $K_0>0$, then
\begin{equation}
\small
    \EE F\left(\theta_0^K,\vtheta^K\right)-F^*={\cal O}\left({1}/{K}\right).
\end{equation}
\end{theorem}

\subsection{Convergence under stochastic unbounded delay}
We make the following assumption \emph{only} for this subsection.
\begin{assumption}[Stochastic unbounded delay]\label{assump:unboundeddelay}
For each client $m$, the delay $\tau_{n_k,m}^k$ is an random variable with unbounded support. 
And there exists $\bar p_m, \rho>0$ such that $\PP\left(\tau_{n_k,m}^k=d\right)\leq \bar p_m\rho^d:=p_{m,d}$.
\end{assumption}

Under Assumption \ref{assump:unboundeddelay}, we obtain the convergence rates of the same order as those the under bounded delay assumption. 
\begin{theorem}\label{thm:nc2}
Under Assumptions \ref{assump:sample}-\ref{assump:client} and \ref{assump:unboundeddelay}, if $\eta_0^k=\eta_m^k=\min\big\{\frac{1}{4(1+\min_m\sqrt{c_m})L}, \frac{c_{\eta}}{\sqrt{K}}\big\}$, then we have
\begin{equation}
\small
    \frac{1}{K}\sum\limits_{k=0}^{K-1}\EE[\|\nabla F(\theta_0^k,\vtheta^k)\|^2]={\cal O}\left({1}/{\sqrt{K}}\right).
\end{equation}
\end{theorem}

Under the additional assumption of strong convexity, the convergence rate is improved.
\begin{theorem}\label{thm:sc2}
Assume that $F$ is $\mu$-strongly convex in $(\theta_0,\vtheta)$. Then under Assumptions \ref{assump:sample}-\ref{assump:client} and \ref{assump:unboundeddelay}, if $\eta_0^k=\eta_m^k=\frac{2}{\nu(k+K_0)}$ where $K_0$ is a positive constant, then it follows that
\begin{equation}
\small
    \EE F(\theta_0^K,\vtheta^K)-F^*={\cal O}\left({1}/{K}\right).
\end{equation}
\end{theorem}

It worth mentioning that under the assumption of bounded delay and unbounded but stochastic delay, without even coordinating clients' gradient samples and local model updates, our algorithm achieves the same order of convergence as that of block-wise SGD in both cases \cite{XuYin2015_block}.


\vspace{-0.1cm}
\section{Perturbed local embedding: Enforcing differential privacy and smoothness}\label{sec:randomneuron}
In this section, we introduce a local perturbation technique that is applied by each client to enforce the differential privacy of local information, which also smoothes the otherwise nonsmooth mapping of local embedding.

\subsection{Local perturbation}
Recall that $h_m$ denotes a local embedding function of client $m$ with the parameter $\theta_m$ which embeds the information of local data $x_{n,m}$ into its outputs $h_{n, m}:=h_m(\theta_m;x_{n,m})$. 
When $h_m$ is linear embedding, it is as simple as $h_m(\theta_m;x_{n,m})=x_{n,m}^{\top}\theta_m$. 
To further account for nonlinear embedding such as neural networks, we represent $h_{n,m}$ in the following composite form
\begin{subequations}\label{eqn:localembed}
\small
\begin{align}
    &u_0=x_{n,m}\\
    &u_l=\sigma_l(w_lu_{l-1}+b_l),~~~l=1,\ldots, L\label{eqn:localem2}\\
    &h_{n,m}=u_L 
\end{align}
\end{subequations}
where $\sigma_l$ is a linear or nonlinear function, and $w_l,b_l$ corresponds to the parameter $\theta_m$ of $h_m$, e.g., $\theta_m:=[w_1,\cdots,w_L, b_1,\cdots,b_L]^{\top}$. 
Here we assume that $\sigma_l$ is $L_{\sigma_l}^0$-Lipschitz continuous.
Specially, when $h_m$ is linear, the composite embedding \eqref{eqn:localembed} corresponds to $L=1, \sigma_1(z)=z$. 

We perturb the local embedding function $h_m$ by adding a random neuron with output $Z_l$ at each layer $l$ (cf. \eqref{eqn:localem2})
\begin{align}\label{eqn:pertlocalembed}
    &u_l=\sigma_l(w_lu_{l-1}+b_l+Z_l),~~~l=1,\ldots, L
\end{align}
where $Z_1,\ldots, Z_L$ are independent random variables. With properly chosen distributions of $Z_l, l=1,\ldots, L$, we show below $h_m$ is smooth and enables differential privacy. While it does not exclude other options, our choice of the perturbation distributions is 
\begin{subequations}\label{eq.pertb}
\small
\begin{align}
        &Z_L\sim\mathcal N\left(0,c^2\right)\\
        &Z_l\sim\mathcal U[-\sqrt{3}c_l,\sqrt{3}c_l],~~~l=1,\cdots, L-1
\end{align}
\end{subequations}
where $\mathcal N\left(0,c^2\right)$ denotes the Gaussian distribution with zero mean and variance $c^2$, and $\mathcal U[-\sqrt{3}c_l,\sqrt{3}c_l]$ denotes the uniform distribution over $[-\sqrt{3}c_l,\sqrt{3}c_l]$.

\begin{figure*}[t]
\vspace{-0.2cm}
\centering
\begin{tabular}{cccc}
\hspace{-0.5cm}
\includegraphics[width=4.8cm]{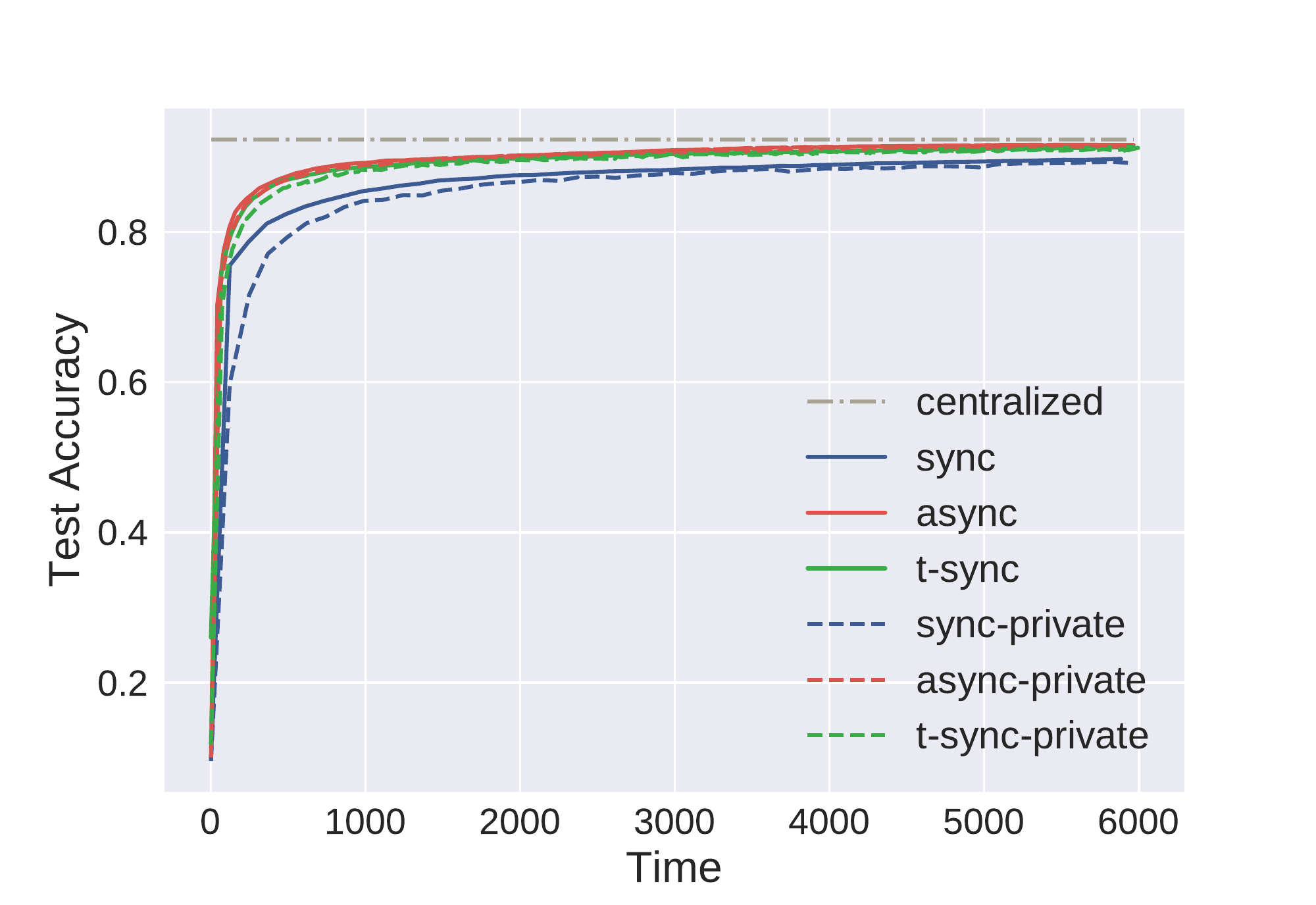}
\hspace{-0.6cm}
\includegraphics[width=4.8cm]{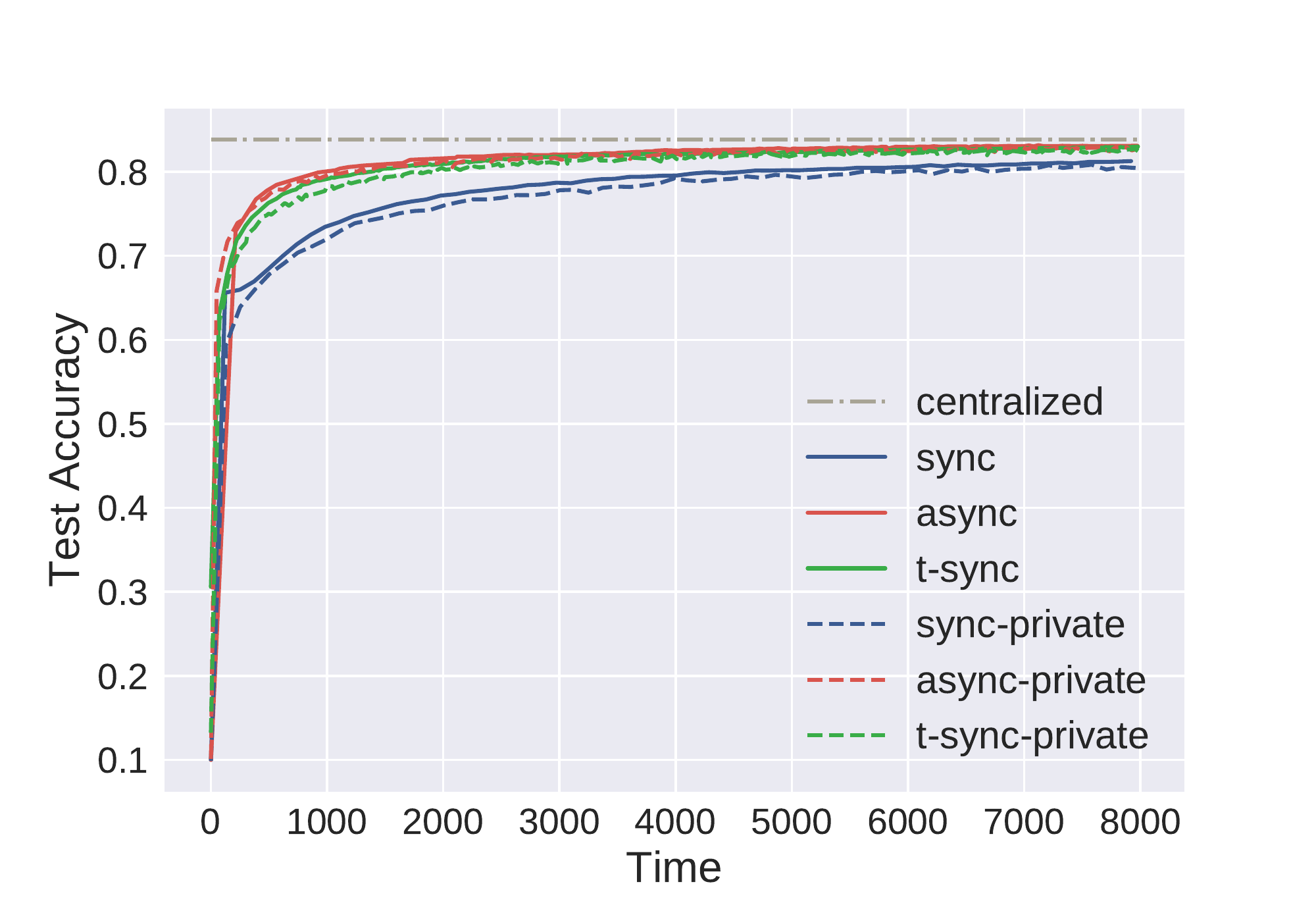}
\hspace{-0.6cm}
\includegraphics[width=4.8cm]{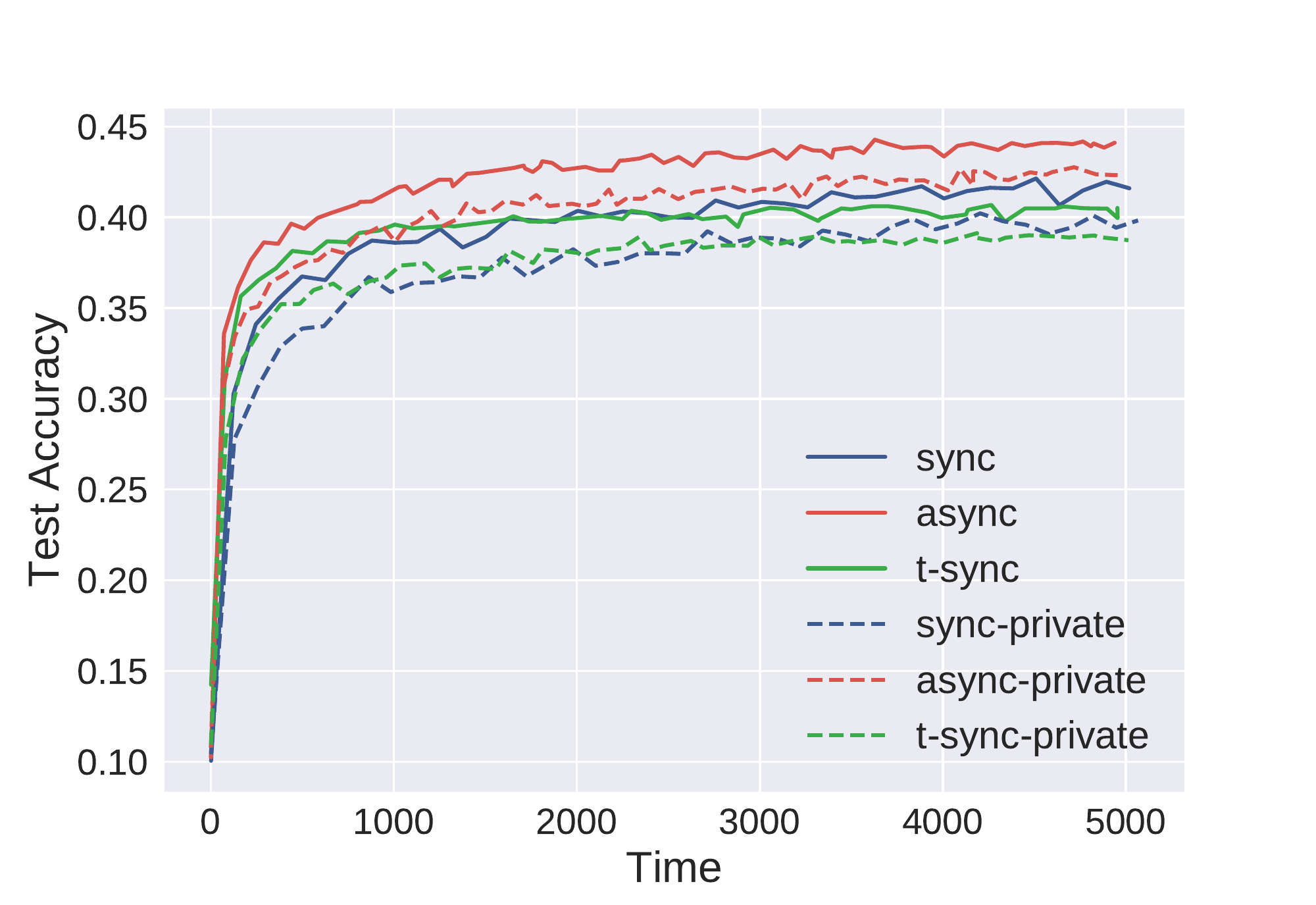}
\hspace{-0.6cm}
\includegraphics[width=4.8cm]{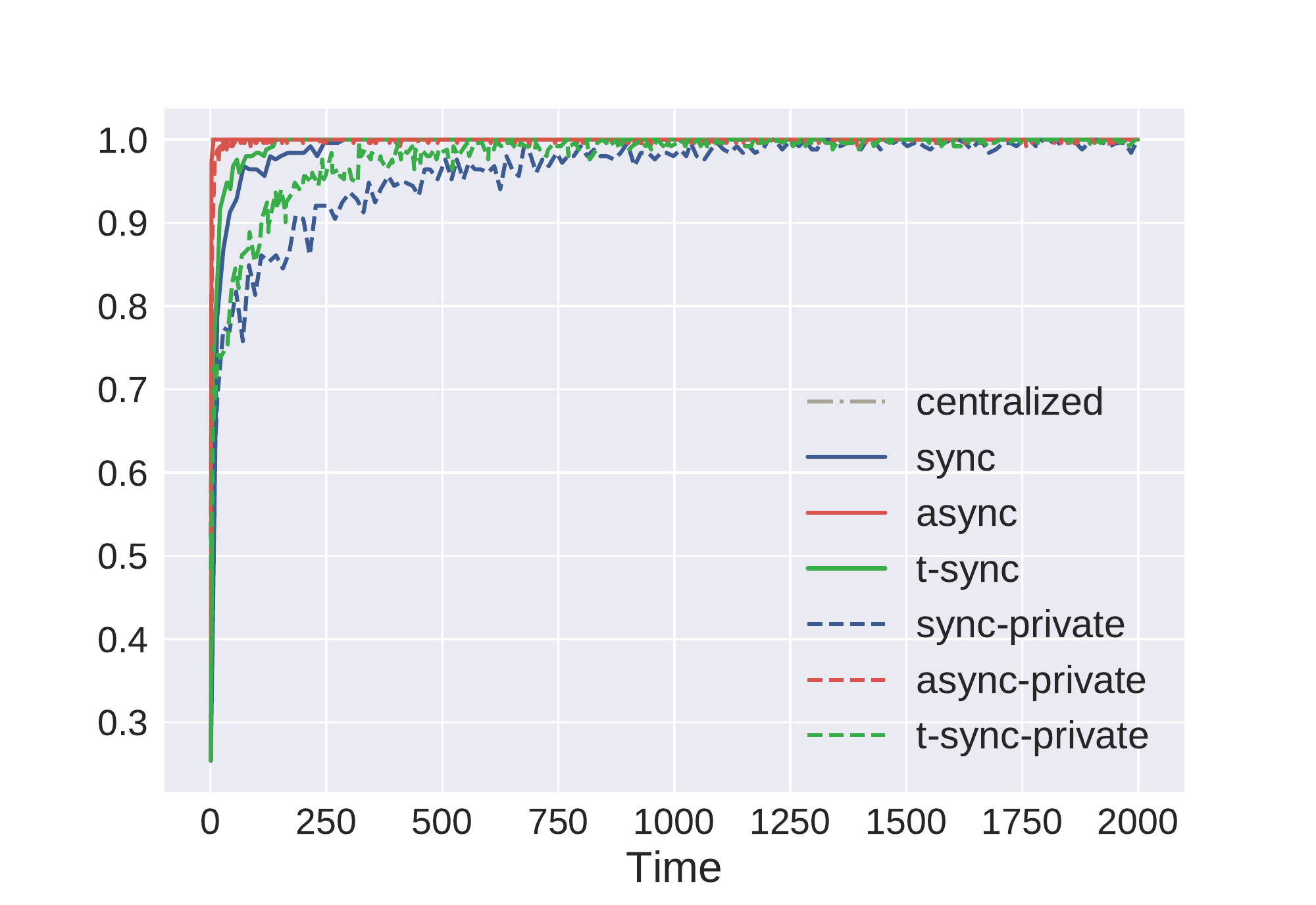}
\end{tabular}
\vspace{-0.5cm}
  \caption{Testing accuracy versus clock time (sec) in MNIST, Fashion-MNIST, CIFAR10 and Parkinson datasets (from left to right).}
\label{fig:test1}
\vspace{-0.2cm}
\end{figure*}


\subsection{Enforcing smoothness}
The convergence results in Section \ref{sec:convergence} hold under Assumption \ref{assump:lip} which requires the smoothness of the overall loss function. 
Inspired by the randomized smoothing technique \cite{duchi2012randomized,nesterov2017random}, we are able to smooth the objective function by taking expectation with respect to the random neurons. Intuitively this follows the fact that the smoothness of a function can be increased by convolving with proper distributions. By adding random neuron $Z_l$, the landscape of $\sigma_l$ will be smoothed in expectation with respect to $Z_l$. And by induction, we can show the smoothness of local embedding vector $h_m$. Then so long as the loss function $\ell$ is smooth w.r.t. the local embedding vector $h_m$, the global objective $F$ is smooth by taking expectation with respect to all the random neurons. 

We formally establish this result in the following theorem.
\begin{theorem}\label{thm:smoothness}
For each embedding function $h_m$, if the activation functions follow $\sigma_l=\sigma, \,\forall l$, and the variances of the random neurons follow \eqref{eq.pertb}, and assume $\|w_l\|$ is bounded, then with $\mathbf{Z}:=[Z_1^{\top},\cdots, Z_L^{\top}]^{\top}$, the perturbed loss satisfies Assumption \ref{assump:lip}, which is given by
\begin{equation}\label{eq.perloss}
\small
    F_c(\theta_0,\vtheta):=\mathbb{E}_{\mathbf{Z}}[F(\theta_0,\vtheta;\mathbf{Z})].
\end{equation}
Starting from $L_{b_L}^{h}={L_{\sigma}^0d}/{c}$, the smoothness constants of the local model $\theta_m$ denoted as $L^{F_c}_{\theta_m}$ satisfy the following recursion ($l=1,\cdots,L-1$)
{\small
\begin{align}\label{eq.smoothness}
    &\!\!L_{b_l}^{h}=L_{b_{l+1}}^{h}\|w_{l+1}\|(L_{\sigma}^0)^2+L_{\sigma}^0\|w_L\|\cdots L_{\sigma}^0\|w_{l+1}\|L_{\bar\sigma}(c_l)\nonumber\\
    &\!\!L_{w_l}^{h}=\EE[\|u_{l-1}\|]L_{b_l}^{h}\nonumber\\
    &\!\!L^{F_c}_{\theta_m}=L^{\ell}_{h_m}(L_{h_m}^0)^2+L_{\ell}^0\sum_{l=1}^L(L_{w_l}^{h}+L_{b_l}^{h})+L^{r}_{\theta_m}
\end{align}}
where $L^{r}_{\theta_m}$ is the smoothness constant of the regularizer w.r.t. $\theta_m$; $L_{b_l}^{h}$ and $L_{w_l}^{h}$ are the smoothness constants of the perturbed local embedding $h$ w.r.t. the bias $b_l$ and weight $w_l$; and $L_{\bar{\sigma}}(c_l):={2\sqrt{d}L_{\sigma}^0}/{c_l}$ is the smoothness constant of the neuron at $l$th layer under the uniform perturbation.
\end{theorem}
\vspace{-0.1cm}
Theorem \ref{thm:smoothness} implies that the perturbed loss is smooth w.r.t. the local model $\theta_m$, and a large perturbation (large $c_l$ or $c$) will lead to a smaller smoothness constant.


\subsection{Enforcing differential privacy}
We now connect the perturbed local embedding technique with the private information exchange in Algorithms \ref{alg:Async-SGD-FL}-\ref{alg:ksync-SGD-FL}. 

As local clients keep sending out embedded information, it is essential to prevent any attacker to trace back to a specific individual via this observation. 
Targeting a better trade-off between the privacy and the accuracy, we leverage the Gaussian differential privacy (GDP) developed in \cite{dong2019gaussian}.
\begin{definition}[\cite{dong2019gaussian}]
A mechanism $\mathcal M$ is said to satisfy $\mu$-GDP if for all neighboring datasets $S$ and $S'$, we have
\begin{equation}\label{def.dp}
    T(\mathcal M(S), \mathcal M(S'))\geq T(\mathcal N(0,1),\mathcal N(\mu, 1))
\end{equation}
where the trade-off function $ T(P,Q)(\alpha)=\inf\{\beta_{\phi}:\alpha_{\phi}\leq\alpha\}$, and $\alpha_{\phi}, \beta_{\phi}$ are type I and II errors given a threshold $\phi$.
\end{definition}
Intuitively, $\mu$-GDP guarantees that distinguishing two adjacent datasets via information revealed by $\mathcal M$ is at least as difficult as distinguishing the two distributions $\mathcal N(0,1)$ and $\mathcal N(\mu, 1)$. Smaller $\mu$ means less privacy loss.

\vspace{-0.1cm}
To characterize the level of privacy of our local embedding approaches, we build on the moments accountant technique originally developed in \cite{abadi2016} to establish that adding random neurons endows Algorithm \ref{alg:Async-SGD-FL} with GDP. 

\begin{theorem}\label{thm:privacy}
Under the same set of assumptions as those in Theorem \ref{thm:smoothness}, for client $m$, if we set the variance of the Gaussian random neuron at the $L$-th layer as
\begin{equation}\label{eqn:privacy}
\small
    c={\cal O}\left({N_m\sqrt{K}}/(\mu N)\right)
\end{equation}
where $N_m$ is the size of minibatch used at client $m$, $N$ is the size of the whole batch, $K$ is the number of queries (i.e. the number of data samples processed by $h_m$ at client $m$), then VAFL satisfies
${\mu}$-GDP for the dataset of client $m$.
\end{theorem}

\vspace{-0.1cm}
Theorem \ref{thm:privacy} demonstrates the trade-off between accuracy and privacy. To increase privacy, i.e., decrease $\mu$ in \eqref{def.dp}, the variance of random neurons needs to be increased (cf. \eqref{eqn:privacy}). However, as the variance of random neurons increases, the variance of the stochastic gradient \eqref{eqn:gradient} also increases, which will in turn lead to slower convergence.


\begin{figure}[t]
\vspace{-0.5cm}
\centering
\includegraphics[width=.38\textwidth]{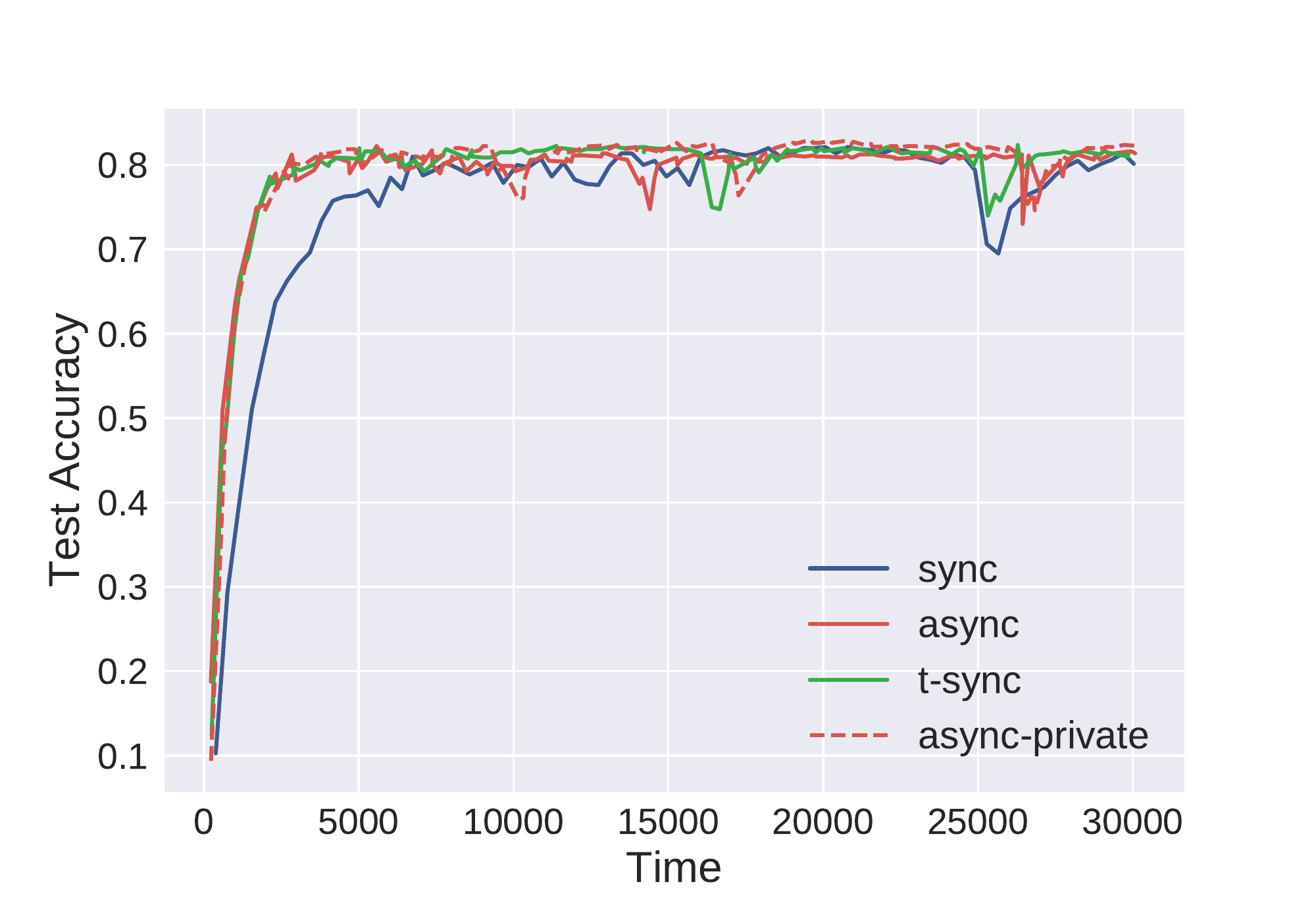}
\vspace{-0.2cm}
\caption{Testing accuracy of VAFL with nonlinear local embedding on \textit{ModelNet40} dataset.}
\label{fig:Nonlinear_acc}
\vspace{-0.3cm}
\end{figure}


\begin{figure}[t]
\vspace{-0.5cm}
\centering
\includegraphics[width=.38\textwidth]{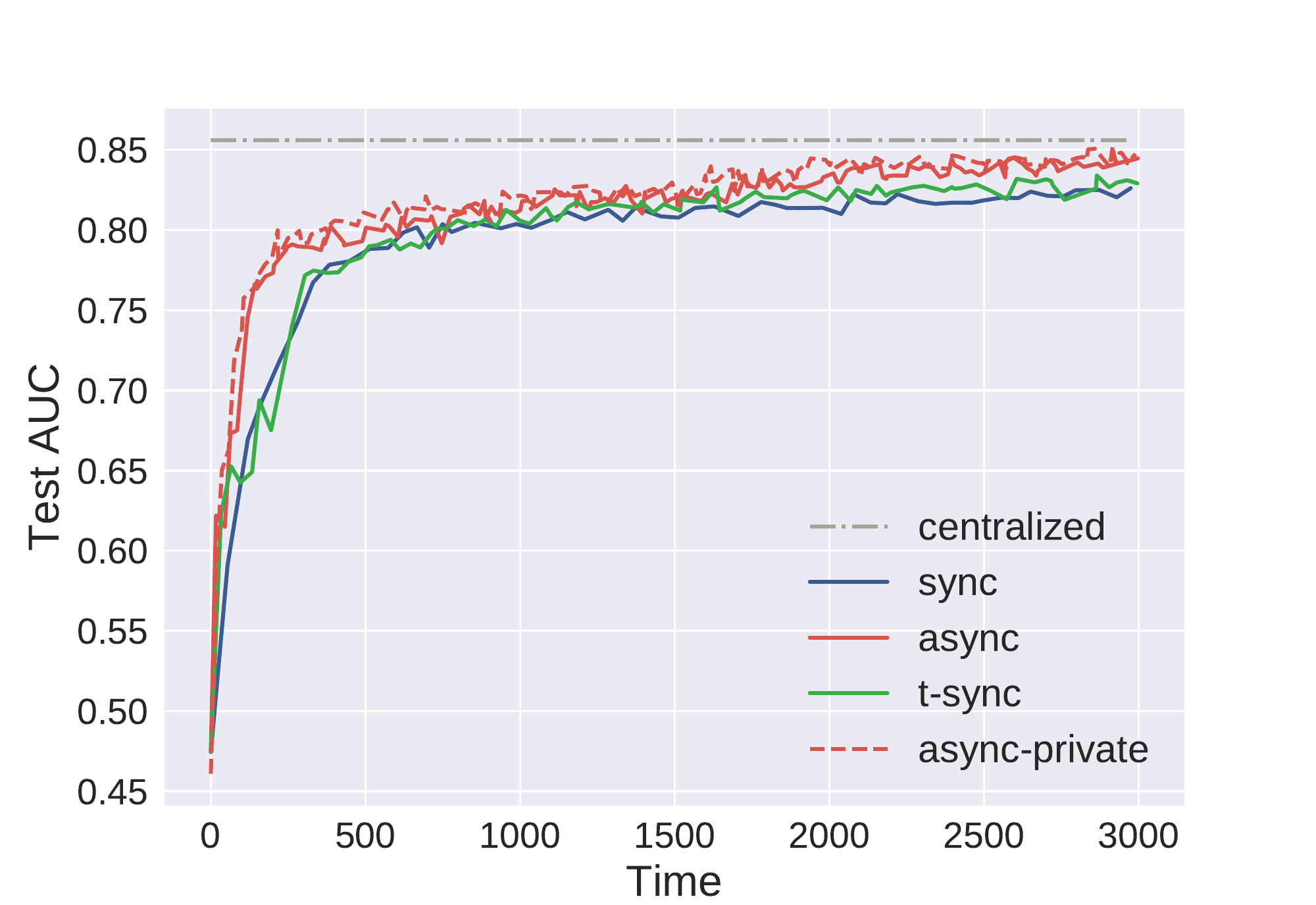}
\vspace{-0.2cm}
\caption{AUC curve of VAFL with local LSTM embedding on \textit{MIMIC-III} clinical care dataset.}
\label{fig:mortality_auc}
\vspace{-0.5cm}
\end{figure}

\section{Numerical tests and remarks}
We benchmark the fully asynchronous version of VAFL ({\bf async}) in Algorithm \ref{alg:Async-SGD-FL}, and $t$-synchronous version of VAFL ({\bf t-sync}) in Algorithm \ref{alg:ksync-SGD-FL} with the synchronous block-wise SGD ({\bf sync}), which requires synchronization and sample index coordination among clients in each iteration. 
We also include private versions of these algorithms via perturbed local embedding technique in Section \ref{sec:randomneuron}. 


\noindent\textbf{VAFL for federated logistic regression.}
We first conduct logistic regression on MNIST, Fashion-MNIST, CIFAR10 and Parkinson disease \cite{sakar2019comparative} datasets. 
The $l_2$-regularizer coefficient is set as $0.001$. 
We select $M=7$ for MNIST and MNIST, $M=8$ for CIFAR10 and $M=3$ for PD dataset. 
The testing accuracy versus wall-clock time is reported in Figures \ref{fig:test1}. The dashed horizontal lines represent the results trained on the centralized (non-federated) model, and the dashed curves represent private variants of considered algorithms with variance $c=0.1$. In all cases, VAFL learns a federated model with accuracies comparable to that of the centralized model that requires collecting raw data. 

\noindent\textbf{VAFL for federated deep learning.}
We first train a neural network modified from MVCNN with 12-view data \cite{wu20153d}. 
We use $M=4$ clients, and each client has 3 views of each object and use a 7-layer CNN as local embedding functions, and server uses a fully connected network to aggregate the local embedding vectors. Results are plotted in Figure \ref{fig:Nonlinear_acc}. 

We also test our VAFL algorithm in MIMIC-III --- an open dataset comprising deidentified health data \cite{johnson2016mimic}. 
We perform the in-hospital mortality prediction as in \cite{Harutyunyan2019} among $M=4$ clients.
Each client uses LSTM as the embedding function. 
In Figure \ref{fig:mortality_auc}, we can still observe that async and $t$-sync VAFL learn a federated model with accuracies comparable to that of the centralized model, and requires less time relative to the synchronous FL algorithm.

\clearpage

\onecolumn
\begin{center}
{\Large \bf Supplementary materials for\\
``VAFL: a Method of Vertical Asynchronous Federated Learning"}
\end{center}

\appendix

In this supplementary document, we first present some supporting lemmas that will be used frequently in this document, and then present the proofs of all the lemmas and theorems in the paper, which is followed by details on our experiments. The content of this supplementary document is summarized as follows. 

\vspace{-1cm}
\addcontentsline{toc}{section}{} 
\part{} 
\parttoc 

\section{Supporting Lemmas}\label{sec.sup_lem}
For notational brevity, we define 
\begin{subequations}
\begin{align}
&G_{0}^k:=\nabla_{\theta_0}F(\theta_0^k, \vtheta^k)\\
& G_m^k:=\nabla_{\theta_m}F(\theta_0^k, \vtheta^k)\\
&g_0^k:=\nabla_{\theta_0}\ell\left(\theta_0^k,h_{n_k,1}^k,\ldots, h_{n_k,M}^k; y_{n_k}\right)\\
& g_m^k:=\nabla_{\theta_m}\ell\left(\theta_0^k,h_{n_k,1}^k,\ldots, h_{n_k,M}^k; y_{n_k}\right)\\
&\hat g_0^k:=\nabla_{\theta_0}\ell(\theta_0,h_{n_k,1}^{k-\tau_{n_k,1}^k},\ldots, h_{n_k,M}^{k-\tau_{n_k,M}^k};y_{n_k})\\
&\hat g_m^k:=\left\{\begin{array}{ll}
    \nabla h_m(\theta_m^k;x_{n_k,m})\nabla_{h_m}\ell(\theta_0^k,h_{n_k,1}^{k-\tau_{n_k,1}^k}\ldots, h_{n_k,M}^{k-\tau_{n_k,M}^k};y_{n_k}) &  \text{if }m=m_k\\
    0 & \text{else}.
\end{array}\right.\\
&\hat\vtheta:=[\theta_1^{k-\tau_{n_k,1}^k};\ldots;\theta_M^{k-\tau_{n_k,M}^k}]\label{eqn:defg}.
\end{align}
\end{subequations}

 To handle the delayed information, we leverage the following Lyapunov function for analyzing VAFL
\begin{equation}\label{eqn:lyapunov1}
    V^k:=F(\theta_0^k, \vtheta^k)+\sum\limits_{d=1}^D\gamma_d\|\vtheta^{k-d+1}-\vtheta^{k-d}\|^2
\end{equation}
where $\{\gamma_d\}$ are a set of constants to be determined later.
\begin{lemma}\label{lemma:l1descent}
Under Assumptions \ref{assump:sample} -- \ref{assump:boundeddelay}, for $\eta_0^k\leq\frac{1}{4L},\eta_m^k\leq\frac{1}{4(L+2\gamma_1)}$, it follows that (with $\gamma_{D+1}=0$)
\begin{align}\label{eqn:l1descent1}
\small
    &\EE V^{k+1}-\EE V^k\leq -\left(\frac{\eta_0^k}{2}-L(\eta_0^k)^2\right)\EE[\|\nabla_{\theta_0}F(\theta_0^k,\vtheta^k)\|^2]\nonumber\\
    &\!-\!\!\sum\limits_{m=1}^M\!q_m\!\left(\!\frac{\eta_{m}^k}{2}-L(\eta_m^k)^2-2\gamma_1(\eta_m^k)^2\!\right)\!\EE[\|\nabla_{\theta_m}F(\theta_0^k,\vtheta^k)\|^2]\nonumber\\
    &\!+\sum\limits_{d=1}^D\left(\!Dc^k\!+\!D\gamma_1\max\limits_m2(\eta_m^k)^2L_m^2\!+\!\gamma_{d+1}\!-\!\gamma_d\right)\nonumber\\
    &\!\times\EE[\|\vtheta^{k+1-d}\!-\!\vtheta^{k-d}\|^2]\nonumber\\
        &\!+\!L(\eta_0^k)^2\sigma_0^2+\sum\limits_{m=1}^Mq_m(L+2\gamma_1)(\eta_m^k)^2\sigma_m^2.
\end{align}
\end{lemma}
If $\{\gamma_d\}$ are chosen properly as specified in the supplementary materials, the first three terms in the right hand side of \eqref{eqn:l1descent1} is negative. By carefully choosing $\{\eta_0^k,\eta_m^k\}$, we can ensure the convergence of Algorithm \ref{alg:Async-SGD-FL}.

We first quantify the descent amount in the objective value. 
\begin{lemma}\label{lemma:Fdescent}
Under Assumptions \ref{assump:sample}-\ref{assump:client}, the iterates $\{\theta_0^k, \vtheta^k\}$ generated by Algorithm \ref{alg:Async-SGD-FL} satisfy
\begin{align}
\label{eqn:Fdescent}
    \EE[F(\theta_0^{k+1}, \vtheta^{k+1})|\Theta^k]
   & \leq F(\theta_0^k, \vtheta^k)+c_k\EE[\|\hat \vtheta^k-\vtheta^k\|^2|\Theta^k]+L(\eta_0^k)^2\sigma_0^2\nonumber\\
    &+\sum\limits_{m=1}^Mq_mL(\eta_m^k)^2\sigma_m^2 -(\frac{\eta_0^k}{2}-L(\eta_0^k)^2)\|G_0^k\|^2-\sum\limits_{m=1}^Mq_m(\frac{\eta_{m}^k}{2}-L(\eta_m^k)^2)\|G_m^k\|^2
\end{align}
\end{lemma}
where $\Theta^k$ is the $\sigma$-algebra generated by $\{\theta_0^0,\vtheta^0, \ldots, \theta_0^k,\vtheta^k\}$, and $c_k$ is defined as 
\begin{equation}\label{eqn:const-c}
c_k:=\left(\frac{\eta_0^k}{2}+L(\eta_0^k)^2\right)L_0^2+\max\limits_m\left(\frac{\eta_{m}^k}{2}+L(\eta_m^k)^2\right)L_m^2.
\end{equation}

\begin{proof}
By Assumption \ref{assump:lip}, we have
\begin{align}\label{sec3.eq1}
    &F(\theta_0^{k+1}, \vtheta^{k+1})\nonumber\\
    =&F(\theta_0^k-\eta_0^k\hat g_0^k, \ldots, \theta_{m_k}^k-\hat g_{m_k}^k, \ldots)\nonumber\\
    \leq &F(\theta_0^k, \vtheta^k)-\eta_0^k\dotp{G_0^k, \hat g_0^k)}-\eta_{m_k}^k\dotp{G_{m_k}^k, \hat g_{m_k}^k}+\frac{L(\eta_0^k)^2}{2}\|\hat g_{0}^k\|^2+\frac{L(\eta_{m_k}^k)^2}{2}\|\hat g_{m_k}^k\|^2\nonumber\\
    \leq & \! F(\theta_0^k, \vtheta^k)\!-\eta_0^k\dotp{G_0^k, g_0^k}\!-\eta_0^k\dotp{G_0^k, \hat g_{0}^k-g_{0}^k}\!-\eta_{m_k}^k\dotp{G_{m_k}^k, g_{m_k}^k}\!-\eta_{m_k}^k\dotp{G_{m_k}^k,\hat g_{m_k}^k-g_{m_k}^k}\!+\frac{L(\eta_0^k)^2}{2}\|\hat g_{0}^k\|^2\!+\!\frac{L(\eta_{m_k}^k)^2}{2}\|\hat g_{m_k}^k\|^2\nonumber\\
    \leq& F(\theta_0^k, \vtheta^k)-\eta_0^k\dotp{G_0^k,g_0^k}-\eta_{m_k}^k\dotp{G_{m_k}^k,g_{m_k}^k}+\frac{\eta_0^k}{2}\|G_0^k\|^2+\frac{\eta_{m_k}^k}{2}\|G_{m_k}^k\|^2+L(\eta_0^k)^2\|g_0^k\|^2 + L(\eta_{m_k}^k)^2\|g_{m_k}^k\|^2\nonumber\\
     & +(\frac{\eta_0^k}{2}+L(\eta_0^k)^2)\|\hat g_{0}^k-g_{0}^k\|^2+(\frac{\eta_{m_k}^k}{2}+L(\eta_{m_k}^k)^2)\|\hat g_{m_k}^k-g_{m_k}^k\|^2.
\end{align}
Note that we have 
\begin{align}\label{sec3.eq2}
    \EE\left[\|g_{m_k}^k\|^2|\Theta^k\right]&=\EE\left[\|g_{m_k}^k-G_{m_k}^k+G_{m_k}^k\|^2|\Theta^k\right]\nonumber\\
    &=\EE\left[\|g_{m_k}^k-G_{m_k}^k\|^2|\Theta^k\right]+2\EE\left[\dotp{g_{m_k}^k-G_{m_k}^k, G_{m_k}^k}|\Theta^k\right]+\|G_{m_k}^k\|^2\nonumber\\
    &=\sigma_{m_k}^2+\|G_{m_k}^k\|^2
\end{align}
where the last equality follows from Assumption \ref{assump:sample}.

First we take expectation on \eqref{sec3.eq1} with respect to $n_k$, conditioned on $\Theta^k$ and $m_k=m$, we have
\begin{align}\label{sec3.eq3}
    \EE[F(\theta_0^{k+1}, \vtheta^{k+1})|m_k=m, \Theta^k]
    \leq& F(\theta_0^k, \vtheta^k)-(\frac{\eta_0^k}{2}-L(\eta_0^k)^2)\|G_0^k\|^2-(\frac{\eta_{m}^k}{2}-L(\eta_m^k)^2)\|G_m^k\|^2+L(\eta_0^k)^2\sigma_0^2+L(\eta_m^k)^2\sigma_m^2\nonumber\\
        & +(\frac{\eta_0^k}{2}+L(\eta_0^k)^2)\EE[\|\hat g_{0}^k-g_{0}^k\|^2|m_k=m]+(\frac{\eta_{m}^k}{2}+L(\eta_m^k)^2)\EE[\|\hat g_{m}^k-g_{m}^k\|^2|m_k=m].
\end{align}
Then taking expectation with respect to $m_k$, it follows that  
\begin{align*}
    \EE[F(\theta_0^{k+1}, \vtheta^{k+1})|\Theta^k]
    \leq &F(\theta_0^k, \vtheta^k)-(\frac{\eta_0^k}{2}-L(\eta_0^k)^2)\|G_0^k\|^2-\sum\limits_{m=1}^Mq_m(\frac{\eta_{m}^k}{2}-L(\eta_m^k)^2)\|G_m^k\|^2+L(\eta_0^k)^2\sigma_0^2+\sum\limits_{m=1}^Mq_mL(\eta_m^k)^2\sigma_m^2\\
    &+(\frac{\eta_0^k}{2}+L(\eta_0^k)^2)L_0^2\EE[\|\hat \vtheta^k-\vtheta^k\|^2|\Theta^k]+\sum\limits_{m=1}^Mq_m(\frac{\eta_{m}^k}{2}+L(\eta_m^k)^2)L_m^2\EE[\|\hat \vtheta^k-\vtheta^k\|^2|m_k=m, \Theta^k]\\
    \leq &F(\theta_0^k, \vtheta^k)-(\frac{\eta_0^k}{2}-L(\eta_0^k)^2)\|G_0^k\|^2-\sum\limits_{m=1}^Mq_m(\frac{\eta_{m}^k}{2}-L(\eta_m^k)^2)\|G_m^k\|^2+L(\eta_0^k)^2\sigma_0^2+\sum\limits_{m=1}^Mq_mL(\eta_m^k)^2\sigma_m^2\\
    &+\underbrace{\left(\left(\frac{\eta_0^k}{2}+L(\eta_0^k)^2\right)L_0^2+\max\limits_m\left(\frac{\eta_{m}^k}{2}+L(\eta_m^k)^2\right)L_m^2\right)}_{:=c^k}\EE[\|\hat \vtheta^k-\vtheta^k\|^2|\Theta^k]
\end{align*}
which completes the proof.
\end{proof}

\section{Convergence under bounded delay}

Recalling the definition of $\hat\vtheta^k$ in \eqref{eqn:defg}, if $\tau_{n_k,m}^k\leq D$, then it can be derived that
\begin{equation}
        \|\hat\vtheta^k-\vtheta^k\|^2\leq\sum\limits_{d=1}^DD\left\|\vtheta^{k+1-d}-\vtheta^{k-d}\right\|^2.
\end{equation}

\subsection{Proof of Lemma \ref{lemma:l1descent}}
Recall the definition of $V^k$, that is
\begin{align*}
    V^k=F(\theta_0^k, \vtheta^k)+\sum\limits_{d=1}^D\gamma_d\|\vtheta^{k+1-d}-\vtheta^{k-d}\|^2
\end{align*}
where we initialize the algorithm with $\vtheta^{-D+1}=\cdots=\vtheta^{-1}=\vtheta^0$.
We first decompose $\|\vtheta^{k+1}-\vtheta^k\|^2$ as
\begin{align}\label{eqn:decompose1}
    \|\vtheta^{k+1}-\vtheta^k\|^2=(\eta_{m_k}^k)^2\|\hat g_{m_k}^k\|^2\leq2(\eta_{m_k}^k)^2\|g_{m_k}^k\|^2+2(\eta_{m_k}^k)^2\|\hat g_{m_k}^k-g_{m_k}^k\|^2.
\end{align}
Taking expectation on both sides of \eqref{eqn:decompose1}, and applying \eqref{sec3.eq2} leads to
\begin{align}\label{eqn:decompose2}
    \EE\left[\|\vtheta^{k+1}-\vtheta^k\|^2|\Theta^k\right]&=\sum\limits_{m=1}^M\EE\left[\|\vtheta^{k+1}-\vtheta^k\|^2|m_k=m,\Theta^k\right]\PP(m_k=m)\nonumber\\
    &\leq \sum\limits_{m=1}^M2q_m(\eta_m^k)^2\|G_m^k\|^2+\sum\limits_{m=1}^M2q_m(\eta_m^k)^2\sigma_m^2+2\max\limits_m(\eta_m^k)^2L_m^2\EE[\|\hat\vtheta^k-\vtheta^k\|^2|\Theta^k].
\end{align}

Following Lemma \ref{lemma:Fdescent} in Appendix \ref{sec.sup_lem} and \eqref{eqn:decompose2}, the Lyapunov function $V^k$ satisfies
\begin{align}\label{eqn:l1descent2}
    \EE[V^{k+1}|\Theta^k]-V^k=&\EE[F(\theta_0^{k+1}, \vtheta^{k+1})|\Theta^k]-F(\theta_0^k, \vtheta^k)+\gamma_1\EE[\|\vtheta^{k+1}-\vtheta^k\|^2|\Theta^k]\nonumber\\
    &+\sum\limits_{d=1}^{D-1}(\gamma_{d+1}-\gamma_d)\|\vtheta^{k+1-d}-\vtheta^{k-d}\|^2-\gamma_D\|\vtheta^{k+1-D}-\vtheta^{k-D}\|^2\nonumber\\
    \leq &-\left(\frac{\eta_0^k}{2}-L(\eta_0^k)^2\right)\|G_0^k\|^2-\sum\limits_{m=1}^Mq_m\left(\frac{\eta_{m}^k}{2}-L(\eta_m^k)^2-2\gamma_1(\eta_m^k)^2\right)\|G_m^k\|^2\nonumber\\
    & +L(\eta_0^k)^2\sigma_0^2+\sum\limits_{m=1}^Mq_m(L+2\gamma_1)(\eta_m^k)^2\sigma_m^2\nonumber\\
    & +\sum\limits_{d=1}^{D-1}\left(Dc^k+D\gamma_1\max\limits_m2(\eta_m^k)^2L_m^2+\gamma_{d+1}-\gamma_d\right)\|\vtheta^{k+1-d}-\vtheta^{k-d}\|^2\nonumber\\
    &+\left(Dc^k+D\gamma_1\max\limits_m2(\eta_m^k)^2L_m^2-\gamma_D\right)\|\vtheta^{k+1-D}-\vtheta^{k-D}\|^2.
\end{align}

Since we choose $\eta_0^k,\eta_{m}^k\leq\bar\eta\leq\frac{1}{4(L+2\gamma_1)}$, it follows that $c^k\leq\frac{3}{2}\bar\eta L^2$. 
By taking expectation on both sides of \eqref{eqn:l1descent2}, we have 
\begin{align}\label{eqn:l1descent}
    \EE V^{k+1}-\EE V^k\leq& -\frac{1}{4}\min\{\eta_0^k,q_m\eta_m^k\}\EE[\|\nabla F(\theta_0^k,\vtheta^k)\|^2]+L(\eta_0^k)^2\sigma_0^2+(2\gamma_1+L)\sum\limits_{m=1}^Mq_m(\eta_m^k)^2\sigma_m^2\nonumber\\
    &-\sum\limits_{d=1}^{D-1}\left(\gamma_d-\gamma_{d+1}-\frac{3}{2}D\bar\eta L^2-2D\gamma_1\bar\eta^2L^2\right)\EE\|\vtheta^{k+1-d}-\vtheta^{k-d}\|^2\nonumber\\
    &-\left(\gamma_D-\frac{3}{2}D\bar\eta L^2-2D\gamma_1\bar\eta^2L^2\right)\EE\|\vtheta^{k+1-D}-\vtheta^{k-D}\|^2.
\end{align}

\subsection{Proof of Theorem \ref{thm:nc1}}
Define $\gamma_1=\frac{\frac{3}{2}\bar\eta D^2L^2}{1-2D^2\bar\eta^2L^2}\leq\frac{1}{2}DL$, and $\eta_0^k=\eta_m^k=\eta=\min\{\frac{1}{4(D+1)L}, \frac{c_{\eta}}{\sqrt{K}}\}$. 
Select $\gamma_2,\ldots,\gamma_D$ as follows
\begin{align*}
\gamma_{d+1}=\gamma_d-\frac{3}{2}D\eta L^2-2D\gamma_1\eta^2L^2,~~~ d=1,\ldots, D-1.
\end{align*}
It can be verified that $\gamma_D-\frac{3}{2}D\eta L^2-2D\gamma_1\eta^2L^2\geq 0$. Then $\eqref{eqn:l1descent}$ reduces to
\begin{align}\label{eqn:l1descent3}
    \EE V^{k+1}-\EE V^k\leq& -\frac{1}{4}\min_m q_m\eta\EE[\|\nabla F(\theta_0^k,\vtheta^k)\|^2]+\eta^2L\sigma_0^2+\eta^2(2\gamma_1+L)\sum\limits_{m=1}^Mq_m\sigma_m^2.
\end{align}
By summing over $k=0,\cdots,K-1$ and using $\eta\leq \frac{c_{\eta}}{\sqrt{K}}$, it follows that
\begin{align*}
    \frac{1}{K}\sum\limits_{k=0}^{K-1}\EE[\|\nabla F(\theta_0^k,\vtheta^k)\|^2]&\leq \frac{F^0-F^*+K\eta^2L\sigma_0^2+K\eta^2(2\gamma_1+L)\sum\limits_{m=1}^Mq_m\sigma_m^2}{\frac{1}{4}\min\limits_mq_m\eta K}\\
    &\leq\frac{16DL(F^0-F^*)}{\min\limits_mq_mK}+\frac{4c_{\eta}(F^0-F^*)}{\min\limits_mq_m\sqrt{K}}+\frac{4c_{\eta}L\sigma_0^2+c_{\eta}(8\gamma_1+4L)\sum\limits_{m=1}^Mq_m\sigma_m^2}{\min\limits_mq_m\sqrt{K}}.
\end{align*}
\subsection{Proof of Theorem \ref{thm:sc1}}
By the $\mu$-strong convexity of $F(\theta_0,\vtheta)$, we have
\begin{align}\label{eqn:strongcvx}
    2\mu(F(\theta_0,\vtheta)-F^*)\leq\|\nabla F(\theta_0,\vtheta)\|^2.
\end{align}
Choose $\gamma_d$ such that
\begin{align*}
    &\gamma_d-\gamma_{d+1}-\frac{3}{2}D\bar\eta L^2-2D\gamma_1\bar\eta^2L^2=\frac{\mu}{2}\min\limits_mq_m\bar\eta\gamma_1, ~~~d=1,\ldots, D-1\\
    &\gamma_D-\frac{3}{2}D\bar\eta L^2-2D\gamma_1\bar\eta L^2=\frac{\mu}{2}\min\limits_m q_m\bar\eta\gamma_1.
\end{align*}
Solve the above linear equations above and get
\begin{align*}
    \gamma_1=\frac{\frac{3}{2}\bar\eta D^2L^2}{1-2D^2\bar\eta^2L^2-\frac{\mu}{2}\min\limits_mq_mD\bar\eta},~~~ \gamma_d=(D+1-d)(\frac{3}{2}D\bar\eta L^2+2D\gamma_1\bar\eta L^2+\frac{\mu}{2}\min\limits_mq_m\bar\eta\gamma_1), ~~~d=1,\ldots, D-1.
\end{align*}
If we choose $\eta_0^k=\eta_m^k=\eta^k\leq\bar\eta\leq\frac{1}{4(D+1)L+2\mu\min\limits_mq_mD}$, $\gamma_1\leq 2\bar\eta D^2L^2$, then $\eqref{eqn:l1descent}$ reduces to
\begin{align*}
    \EE V^{k+1}\leq(1-\frac{\mu}{2}\eta^k\min\limits_mq_m)\EE V^k+(\eta^k)^2\left(L\sigma_0^2+(2\gamma_1+L)\sum\limits_{m=1}^Mq_m\sigma_m^2\right).
\end{align*}
Defining $R:=\big(L\sigma_0^2+(2\gamma_1+L)\sum\limits_{m=1}^Mq_m\sigma_m^2\big)$ and $\eta^k=\frac{4}{\mu\min\limits_mq_m(k+K_0)}$, where $K_0=\frac{4(4(D+1)L+2\mu\min\limits_mq_mD)}{\mu\min\limits_mq_m}$, we have
\begin{align*}
    \EE V^k&\leq V^0\prod\limits_{k=0}^{K-1}(1-\frac{\mu}{2}\min\limits_mq_m\eta^k)+R\sum\limits_{k=0}^{K-1}(\eta^k)^2\prod\limits_{j=k+1}^{K-1}(1-\frac{\mu}{2}\min\limits_mq_m\eta^j)\\
    &=V^0\prod\limits_{k=0}^{K-1}\frac{k+K_0-2}{k+K_0}+\frac{16R}{\mu^2\min\limits_mq_m}\sum\limits_{k=0}^{K-1}\frac{1}{(k+K_0)^2}\prod\limits_{j=k+1}^{K-1}\frac{j+K_0-2}{j+K_0}\\
    &\leq \frac{(K_0-2)(K_0-1)}{(K+K_0-2)(K+K_0-1)}V^0+\frac{16R}{\mu^2\min\limits_mq_m}\sum\limits_{k=0}^{K-1}\frac{1}{(k+K_0)^2}\frac{(k+K_0-1)(k+K_0)}{(K+K_0-2)(K+K_0-1)}\\
    &\leq\frac{(K_0-1)^2}{(K+K_0-1)^2}(F(\theta_0^0,\vtheta^0)-F^*)+\frac{16RK}{\mu^2\min\limits_mq_m(K+K_0-1)^2}.
\end{align*}
\section{Convergence under stochastic unbounded delay}
We first present a useful fact. 
Given the definition of $\bar p_m, p_{m,d}$ in Assumption \ref{assump:unboundeddelay}, it can be shown that
\begin{align*}
    &\sum\limits_{s=d}^{\infty}sp_{m,s}=\bar p_m\left(\frac{d\rho^d}{1-\rho}+\frac{\rho^{d+1}}{(1-\rho)^2}\right):=c_{m,d}\\
    &\sum\limits_{s=d}^{\infty}c_{m,s}=\bar p_m\left(\frac{d\rho^d}{(1-\rho)^2}+\frac{2\rho^{d+1}}{(1-\rho)^3}\right)\\
    &\sum\limits_{d=1}^{\infty}c_{m,d}=\bar p_m\left(\frac{\rho}{(1-\rho)^2}+\frac{2\rho^2}{(1-\rho)^3}\right):=c_m.
\end{align*}

For unbounded delay, we have the following relation 
\begin{align*}
    \EE[\|\hat\vtheta^k-\vtheta^k\|^2|\Theta^k]
    &=\sum\limits_{m=1}^M\EE[\|\theta_m^{k-\tau_{n_k,m}^k}-\theta_m^k\|^2|\Theta^k]\\
    &=\sum\limits_{m=1}^M\sum\limits_{s=1}^{\infty}\EE[\|\theta_m^{k-s}-\theta_m^k\||\Theta^k]\PP(\tau_{n_k,m}^k=s)\\
    &\leq\sum\limits_{m=1}^M\sum\limits_{s=1}^{\infty}\sum\limits_{d=1}^ssp_{m,s}\|\theta_m^{k+1-d}-\theta_m^{k-d}\|^2\\
    &=\sum\limits_{m=1}^M\sum\limits_{d=1}^{\infty}c_{m,d}\|\theta_m^{k+1-d}-\theta_m^{k-d}\|^2.
\end{align*}

Similar to \eqref{eqn:decompose2}, we can decompose the difference term as
\begin{align}\label{eqn:decompose3}
    \EE\left[\|\theta_m^{k+1}-\theta_m^k\|^2|\Theta^k\right]\leq 2q_m(\eta_m^k)^2\|G_m^k\|^2+2q_m(\eta_m^k)^2\sigma_m^2+2q_m(\eta_m^k)^2L^2\EE[\|\hat\theta_m^k-\theta_m^k\|^2|\Theta^k].
\end{align}
Following Lemma \ref{lemma:Fdescent} and \eqref{eqn:decompose3}, we have
\begin{align*}
    &\EE[V^{k+1}|\Theta^k]-V^k\\
    =&\EE[F(\theta_0^{k+1}, \vtheta^{k+1})|\Theta^k]-F(\theta_0^k, \vtheta^k)+\sum\limits_{m=1}^M\gamma_{m,1}\EE[\|\theta_m^{k+1}-\theta_m^k\|^2|\Theta^k]+\sum\limits_{m=1}^M\sum\limits_{d=1}^{\infty}(\gamma_{d+1}-\gamma_d)\|\theta_m^{k+1-d}-\theta_m^{k-d}\|^2\\
    \leq &-\left(\frac{\eta_0^k}{2}-L(\eta_0^k)^2\right)\|G_0^k\|^2-\sum\limits_{m=1}^Mq_m\left(\frac{\eta_m^k}{2}-(2\gamma_{m,1}+L)(\eta_m^k)^2\right)\|G_m^k\|^2+L(\eta_0^k)^2\sigma_0^2\\
    &+\sum\limits_{m=1}^Mq_m(\eta_m^k)^2(2\gamma_{m,1}+L)\sigma_m^2 +\sum\limits_{m=1}^M\sum\limits_{d=1}^{k}\left(\left(c^k+2q_m\gamma_{m,1}(\eta_m^k)^2L^2\right)c_{m,d}+\gamma_{m,d+1}-\gamma_{m,d}\right)\|\theta_m^{k+1-d}-\theta_m^{k-d}\|^2.
\end{align*}

If we choose $\eta_0^k,\eta_m^k\leq\bar\eta\leq\frac{1}{4(L+2\max_m\gamma_{m,1})}$, then $c^k\leq 
\frac{3}{2}\bar\eta L^2$. By direct calculation, we have
\begin{align}\label{eqn-pf:lemma3}
    \EE V^{k+1}-\EE V^k\leq&-\frac{1}{4}\min\{\eta_0^k,q_m\eta_m^k\}\EE[\|\nabla F(\theta_0^k,\vtheta^k)\|^2]+L(\eta_0^k)^2\sigma_0^2+\sum\limits_{m=1}^Mq_m(\eta_m^k)^2(2\gamma_{m,1}+L)\sigma_m^2\nonumber\\
    &-\sum\limits_{m=1}^M\sum\limits_{d=1}^{\infty}\left(\gamma_{m,d}-\gamma_{m,d+1}-c_{m,d}\left(\frac{3}{2}\bar\eta L^2+2q_m\gamma_{m,1}\bar\eta^2L^2\right)\right)\EE[\|\theta_m^{k+1-d}-\theta_m^{k-d}\|^2].
\end{align}

If we select $\gamma_{m,d}$ such that 
\begin{align*}
    \left(\frac{3}{2}\bar\eta L^2+2q_m\gamma_{m,1}\bar\eta^2L^2\right)c_{m,d}+\gamma_{m,d+1}-\gamma_{m,d}=-\xi_m c_{m,d}, ~~~m=1,\ldots, M, ~d=1,\ldots, \infty
\end{align*}
then it remains that
\begin{align*}
    \gamma_{m,d}=\sum\limits_{s=d}^{\infty}c_{m,s}\left(\frac{3}{2}\bar\eta L^2+2q_m\gamma_{m,1}\bar\eta^2L^2+\xi_m\right).
\end{align*}

\subsection{Proof of Theorem \ref{thm:nc2}}
We set the parameters as
\begin{equation}
\xi_m=0,\qquad   \qquad c_m=\sum\limits_{d=1}^{\infty}c_{m,d},\qquad\qquad \gamma_{m,1}=\frac{\frac{3}{2}c_m\bar\eta L^2}{1-2c_mq_m\bar\eta^2L^2}\leq2\bar\eta c_mL^2\leq\frac{1}{2}\sqrt{c_m}L 
\end{equation}
and 
\begin{equation}
    \eta_0^k=\eta_m^k=\eta=\min\left\{\frac{1}{4(1+\max_m\sqrt{c_m})L}, \frac{c_{\eta}}{\sqrt{K}}\right\}.
\end{equation}

Plugging these constants into \eqref{eqn-pf:lemma3}, we have
\begin{align}\label{eqn:pf-thm3}
    \EE V^{k+1}-\EE V^k\leq\frac{1}{4}\min\limits_mq_m\eta\EE[\|\nabla F(\theta_0^k,\vtheta^k)\|^2]+\eta^2\sum\limits_{m=1}^Mq_m(2\gamma_{m,1}+L)\sigma_m^2.
\end{align}

By summing \eqref{eqn:pf-thm3} over $k=0,\cdots,K-1$, it follows that
\begin{align*}
    \frac{1}{K}\sum\limits_{k=0}^{K-1}\EE[\|\nabla F(\theta_0^k,\vtheta^k)\|^2]&\leq\frac{F^0-F^*+K\eta^2\sum\limits_{m=1}^M(2\gamma_{m,1}+L)q_m\sigma_m^2}{\frac{1}{4}\min\limits_mq_m\eta K}\\
    &\leq\frac{16(1+\max_m\sqrt{c_m})L(F^0-F^*)}{\min\limits_m q_mK}+\frac{4c_{\eta}(F^0-F^*)}{\min\limits_mq_m\sqrt{K}}+\frac{4c_{\eta}\sum\limits_{m=1}^M(2\gamma_{m,1}+L)q_m\sigma_m^2}{\min\limits_mq_m\sqrt{K}}.
\end{align*}

\subsection{Proof of Theorem \ref{thm:sc2}}
If we set $\xi_m=\frac{1}{4}c_m\bar\eta L^2$ and $\eta^k\leq\bar\eta=\frac{1}{4(1+\max_m\sqrt{c_m})L}$, then
\begin{equation*}
\gamma_{m,1}=\frac{\frac{3}{2}c_m\bar\eta L^2+\xi_m}{1-2c_m\bar\eta^2L^2}= 2c_m\bar\eta L^2\leq\frac{1}{2}\sqrt{c_m}L.
\end{equation*}
Plugging the parameters in \eqref{eqn-pf:lemma3} and using the strong convexity in \eqref{eqn:strongcvx}, we have
\begin{align*}
    \EE V^{k+1}\leq(1-\nu\eta^k)\EE V^k+(\eta^k)^2R
\end{align*}
where $\nu=\inf\limits_{m,d}\{\frac{\xi_m c_{m,d}}{\bar\eta\gamma_{m,d}},\frac{\mu q_m}{2}\}$ and $R:=\big(L\sigma_0^2+(2\gamma_{m,1}+L)\sum\limits_{m=1}^Mq_m\sigma_m^2\big)$.

Choosing $\eta^k=\frac{2}{\nu(k+K_0)}$ with $K_0=\frac{4(1+\max_m\sqrt{c_m})L}{\nu}$, it follows that
\begin{align*}
    \EE V^k&\leq\prod\limits_{k=0}^{K-1}(1-\nu\eta^k)V^0+R\sum\limits_{k=0}^{K-1}(\eta^k)^2\prod\limits_{j=k+1}^{K-1}(1-\nu\eta^j)\\
    &=V^0\prod\limits_{k=0}^{K-1}\frac{k+K_0-2}{k+K_0}+\frac{16R}{\mu^2\min\limits_mq_m}\sum\limits_{k=0}^{K-1}\frac{1}{(k+K_0)^2}\prod\limits_{j=k+1}^{K-1}\frac{j+K_0-2}{j+K_0}\\
    &\leq \frac{(K_0-2)(K_0-1)}{(K+K_0-2)(K+K_0-1)}V^0+\frac{16R}{\mu^2\min\limits_mq_m}\sum\limits_{k=0}^{K-1}\frac{1}{(k+K_0)^2}\frac{(k+K_0-1)(k+K_0)}{(K+K_0-2)(K+K_0-1)}\\
    &\leq\frac{(K_0-1)^2}{(K+K_0-1)^2}(F^0-F^*)+\frac{16RK}{\mu^2\min\limits_mq_m(K+K_0-1)^2}.
\end{align*}

\begin{remark}
To verify the existence of $\nu>0$, we have
\begin{align*}
    \frac{\xi_m c_{m,d}}{\bar\eta\gamma_{m,d}}&=\frac{c_{m,d}}{\sum\limits_{s=d}^{\infty}c_{m,s}}\frac{1}{\bar\eta}\frac{\xi_m}{\frac{3}{2}\bar\eta L^2+2q_m\gamma_{m,1}\bar\eta^2L^2+\xi_m}
    \geq\frac{1}{\bar\eta}\frac{(1-\rho)\xi_m}{3\bar\eta L^2+4q_m\gamma_{m,1}\bar\eta^2L^2+2\xi_m}\\
    &\geq\frac{1}{\bar\eta}\frac{(1-\rho)\xi_m}{5\bar\eta L^2+\xi_m}=\frac{1}{\bar\eta}\frac{(1-\rho)c_m}{20+c_m}
\end{align*}
where we use the fact that $\frac{c_{m,d}}{\sum\limits_{s=d}^{\infty}c_{m,s}}\geq\frac{1-\rho}{2}$. Then $\nu=\min\{\frac{1}{\bar\eta}\frac{(1-\rho)c_m}{20+c_m},\frac{\mu q_m}{2}\}$.
\end{remark}
\section{Convergence results of vertical $t$-synchronous federated learning}
In the $t$-synchronous, we use ${\cal M}^k$ to denote the set of clients that upload at iteration $k$. For notational brevity, we define 
\begin{subequations}
\begin{align*}
 &   \hat g_0^k=\frac{1}{t}\sum\limits_{m\in\mathcal M^k}\nabla_{\theta_0}\ell(\theta_0,h_{n_k,1}^{k-\tau_{n_k(m),1}^k},\ldots, h_{n_k(m),M}^{k-\tau_{n_k(m),M}^k};y_{n_k})\\
  &  \hat g_m^k=\left\{\begin{array}{ll}
\nabla h_m(\theta_m^k;x_{n_k(m),m})\nabla_{h_m}\ell(\theta_0^k,h_{n_k(m),1}^{k-\tau_{n_k(m),1}^k},\ldots, h_{n_k(m),M}^{k-\tau_{n_k(m),M}^k};y_{n_k(m)}), & \text{if } m\in\mathcal M^k;\\
0, & \text{else.}
\end{array}\right.
\end{align*}
\end{subequations}
Similar to Assumption \ref{assump:client}, we assume that
\begin{assumption}\label{assump:tclient}
The probability of client $m$ in the set of uploading clients ${\cal M}_k$ at iteration $k$ is independent of ${\cal M}^{k-1}, \cdots, {\cal M}^1$, and it satisfies 
\begin{align*}
    \PP(m\in\mathcal M^k):=q_m.
\end{align*}
\end{assumption}
\subsection{Connecting with asynchronous case}
Similar to the previous analysis, the objective value satisfies the following inequality
\begin{align*}
F(\theta_0^{k+1},\vtheta^{k+1})
&\leq F(\theta_0^k,\vtheta^k)+\dotp{G_0^k,\theta_m^{k+1}-\theta_0^k}+\dotp{G_m^k, \theta_m^{k+1}-\theta_m^k}+\frac{L}{2}\|\theta_0^k-\vtheta^k\|^2+\sum\limits_{m=1}^M\frac{L}{2}\|\theta_m^{k+1}-\theta_m^k\|^2\\
&=F(\theta_0^k,\vtheta^k)+\dotp{G_0^k, \hat g_0^k}+\sum\limits_{m\in\mathcal M^k}\dotp{G_m^k, \hat g_m^k}+\frac{L(\eta_0^k)^2}{2}\|\hat g_0^k\|^2+\frac{L}{2}\sum\limits_{m\in\mathcal M^k}(\eta_m^k)^2\|\hat g_m^k\|^2\\
&\leq F(\theta_0^k,\vtheta^k)-\eta_0^k\dotp{G_0^k,g_0^k}-\sum\limits_{m\in\mathcal M^k}\eta_m^k\dotp{G_m^k,g_m^k}+\frac{\eta_0^k}{2}\|G_0^k\|^2+\sum\limits_{m\in\mathcal M}\frac{\eta_m^k}{2}\|G_m^k\|^2+L(\eta_0^k)^2\|G_0^k\|^2\\
&~~~~~+\sum\limits_{m\in\mathcal M^k}L(\eta_m^k)^2\|G_m^k\|^2+(\frac{\eta_0^k}{2}+L(\eta_0^k)^2)\|\hat g_0^k-g_0^k\|^2+\sum\limits_{m\in\mathcal M^k}(\frac{\eta_m^k}{2}+L(\eta_m^k)^2)\|\hat g_m^k-g_m^k\|^2.
\end{align*}
And by taking expectation with respect to $\mathcal M^k, n_k(m)$, it follows that (with $t=\sum\limits_{m=1}^Mq_m$)
\begin{align*}
&\EE[F(\theta_0^{k+1}, \vtheta^{k+1})|\Theta^k]\\
    \leq &F(\theta_0^k, \vtheta^k)-(\frac{\eta_0^k}{2}-L(\eta_0^k)^2)\|G_0^k\|^2-\sum\limits_{m=1}^Mq_m(\frac{\eta_{m}^k}{2}-L(\eta_m^k)^2)\|G_m^k\|^2+\frac{1}{t^2}L(\eta_0^k)^2\sigma_0^2+\sum\limits_{m=1}^Mq_mL(\eta_m^k)^2\sigma_m^2\\
    &+\left(\left(\frac{\eta_0^k}{2}+L(\eta_0^k)^2\right)L_0^2+t\max\limits_m\left(\frac{\eta_{m}^k}{2}+L(\eta_m^k)^2\right)L_m^2\right)\EE[\|\hat \vtheta^k-\vtheta^k\|^2|\Theta^k].
\end{align*}

Following the Lyapunov analysis of the asynchronous case, it can be shown that the vertical $t$-synchronous federated learning achieves the same order of convergence rate as in Theorems \ref{thm:nc1}-\ref{thm:sc2}. 

\subsection{Convergence results}
For completeness, we state the convergence results for the vertical $t$-synchronous federated learning as follows.  
\begin{theorem}[Bounded delay, nonconvex]\label{thm:nc1t}
Under Assumptions \ref{assump:sample},\ref{assump:lip},\ref{assump:boundeddelay} and \ref{assump:tclient}, if $\eta_0^k=t\eta_m^k=\min\{\frac{1}{4(1+D)L}, \frac{c_{\eta}}{\sqrt{K}}\}$ with $c_{\eta}\!>\!0$, then we have
\begin{equation}
    \frac{1}{K}\sum\limits_{k=0}^{K-1}\EE[\|\nabla F(\theta_0^k,\vtheta^k)\|^2]={\cal O}\left({1}/{\sqrt{K}}\right).
\end{equation}
\end{theorem}
\begin{theorem}[Bounded delay, strongly convex]\label{thm:sc1t}
Assume that $F$ is $\mu$-strongly convex in $(\theta_0,\vtheta)$. Then under the same assumptions of Theorem \ref{thm:nc1t}, if $\eta^k=\frac{4}{\mu\min\limits_m\sqrt{q_m}(k+K_0)}$ with $K_0=\frac{4(4(D+1)L+\mu\min\limits_m\sqrt{q_m}D)}{\mu t\min\limits_m\sqrt{q_m}}$, then
\begin{equation}
    \EE F\left(\theta_0^K,\vtheta^K\right)-F^*={\cal O}\left({1}/{K}\right).
\end{equation}
\end{theorem}

\begin{theorem}[Unbounded stochastic delay, nonconvex]\label{thm:nc2t}
Under Assumptions \ref{assump:sample},\ref{assump:lip},\ref{assump:unboundeddelay} and \ref{assump:tclient}, if we choose $\eta_0^k=t\eta_m^k=\min\big\{\frac{1}{4(1+\min_m\sqrt{c_m})L}, \frac{c_{\eta}}{\sqrt{K}}\big\}$, then we have
\begin{equation}
    \frac{1}{K}\sum\limits_{k=0}^{K-1}\EE[\|\nabla F(\theta_0^k,\vtheta^k)\|^2]={\cal O}\left({1}/{\sqrt{K}}\right).
\end{equation}
\end{theorem}

\begin{theorem}[Unbounded stochastic delay, strongly convex]\label{thm:sc2t}
Assume that $F$ is $\mu$-strongly convex in $(\theta_0,\vtheta)$. Then under the same assumptions of Theorem \ref{thm:nc2t}, if $\eta_0^k=\eta_m^k=\frac{2}{\nu(k+K_0)}$ where $K_0=\frac{4(1+\max_m\sqrt{c_m})L}{t\nu}$ and $\nu$ is a positive constant depending on $\mu, L, \bar p_m, \rho$, then it follows that
\begin{equation}
    \EE F\left(\theta_0^K,\vtheta^K\right)-F^*={\cal O}\left({1}/{K}\right).
\end{equation}
\end{theorem}

\section{Proof of Theorem \ref{thm:smoothness}}
Before proceeding to the proof of Theorem \ref{thm:smoothness}, we first present the smoothness of a single neuron in the following lemma.
\begin{lemma}\label{lemma:singlesmoothness}
If $\sigma(x)$ is $L_{\sigma}^0$-Lipschitz continuous and differentiable almost everywhere, $Z$ is a continuous random variable with pdf $\mu(Z)$, then $\bar\sigma(x):=\EE\sigma(x+Z)$ is differentiable with Lipschitz continuous gradient $\nabla\bar\sigma(x)=\EE\nabla\sigma(x+Z)$. 
\end{lemma}
\begin{proof}
We first prove that $\bar\sigma(x)$ is smooth and $\EE\nabla\sigma(x+Z)=\nabla\bar\sigma(x)$.
\begin{align*}
    \frac{\EE_Z\sigma(x+\delta v+Z)-\EE_Z\sigma(x+Z)}{\delta}=\int_{\RR^d}\frac{\sigma(x+\delta v+Z)-\sigma(x+Z)}{\delta}\mu(Z)dZ
\end{align*}
Since $\sigma$ is differentiable almost everywhere, for any fixed $x\in\RR^d$ and directional vector $v\in\RR^d$, we have
\begin{align*}
    \lim_{\delta\rightarrow 0}\frac{\sigma(x+\delta v+Z)-\sigma(x+Z)}{\delta}=v^{\top}\nabla\sigma(x+Z)=\sum\limits_{i=1}^n\frac{\partial \sigma}{\partial x_i}(x+Z)v_i\quad\quad a.e.
\end{align*}
and
\begin{align*}
    \int_{\RR^d}\left|\frac{\sigma(x+\delta v+Z)-\sigma(x+Z)}{\delta}\right|\mu(Z)dZ\leq\int_{\RR^d}L_{\sigma}^0\mu(Z)dZ=L_{\sigma}^0.
\end{align*}
Then by dominated convergence theorem, when taking $\delta\rightarrow 0$, it follows that
\begin{align*}
&\frac{\partial\bar\sigma}{\partial x_i}(x)=\int_{\RR^d}\frac{\partial\sigma}{\partial x_i}(x+Z)\mu(Z)dZ,\\
&\frac{\partial\bar\sigma}{\partial v}(x)=\int_{\RR^d}\sum\limits_{i=1}^d\frac{\partial \sigma}{\partial x_i}(x+Z)v_i\mu(Z)dZ=\sum\limits_{i=1}^n\frac{\partial\bar\sigma}{\partial x_i}(x)v_i.
\end{align*}
Therefore, $\bar\sigma(x)$ is differentiable, that is
\begin{align*}
    \nabla\bar\sigma(x)=\int_{\RR^d}\nabla\sigma(x+Z)\mu(Z)dZ=\EE_Z\nabla\sigma(x+Z).
\end{align*}
Next we derive the smoothness constant of $\bar\sigma(x)$. We focus on the uniform distribution and the Gaussian distribution. 

\textbf{Case I.} Assume that the uniform distribution $Z\sim\mathcal U[-\frac{c}{2},\frac{c}{2}]^d$, i.e., $\mu(Z)=\frac{1}{c^d}\mathbbm{1}_{\{-\frac{c}{2}\leq Z_i\leq \frac{c}{2},1\leq i\leq d\}}(Z)$.
\begin{align*}
    \|\nabla\bar\sigma(x)-\nabla\bar\sigma(x')\|&=\left\|\int_{\RR^d}\nabla\sigma(x+y)\mu(Z)dZ-\int_{\RR^d}\nabla\sigma(x'+Z)\mu(Z)dZ\right\|\\
    &=\left\|\int_{\RR^d}\nabla\sigma(y)(\mu(y-x)-\mu(y-x'))dy\right\|\leq L_{\sigma}^0\int_{\RR^d}\left|\mu(y-x)-\mu(y-x')\right|dy\\
    &\leq \frac{2\sqrt{d}L_{\sigma}^0}{c}\|x-x'\|:=L_{\bar{\sigma}}(c)\|x-x'\|
\end{align*}
where the smoothness constant is defined as $L_{\bar{\sigma}}(c):=\frac{2\sqrt{d}L_{\sigma}^0}{c}$.

\textbf{Case II.} Assume that the Gaussian distribution $Z\sim\mathcal N(0, c^2I_d)$, i.e., $\mu(Z)=\frac{1}{\sigma\sqrt{2\pi}}e^{-\frac{\|Z\|^2}{2\sigma^2}}$. 
\begin{align*}
    \bar\sigma(x)=\int_{\RR^d}\sigma(x+Z)\mu(Z)dZ=\int_{Z\in\RR^d}\sigma(y)\mu(y-x)dy.
\end{align*}
By the Leibniz rule, we have
\begin{align*}
    \nabla\bar{\sigma}(x)=-\int_{\RR^d}\sigma(y)\nabla\mu(y-x)dy=-\int_{\RR^d}\sigma(y+x)\nabla\mu(y)dy.
\end{align*}
Then it follows
\begin{align*}
    \|\nabla\bar\sigma(x)-\nabla\bar\sigma(x')\|&=\left\|\int_{\RR^d}(\sigma(y+x)-\sigma(y+x'))\nabla\mu(y)dy\right\|\\
    &\leq L_{\sigma}^0\left(\int_{\RR^d}\|\nabla\mu(y)\|dy\right)\|x-x'\|\\
    &=\frac{L_{\sigma}^0d}{c}\|x-x'\|:= L_{\bar{\sigma}}(c)\|x-x'\|
\end{align*}
where the smoothness constant is $L_{\bar{\sigma}}(c):=\frac{L_{\sigma}^0d}{c}$.
\end{proof}

\subsection{Proof of Theorem \ref{thm:smoothness}}
Building upon Lemma \ref{lemma:singlesmoothness}, we next prove Theorem \ref{thm:smoothness}. 
For simplicity, we assume that all the activation functions are same, e.g., $\sigma_l=\sigma,\,\forall l=1,\cdots,L$.
We use $L_{f}$ to denote the lipschitz constant of a function $f$. 
In the following proof, we change the order of differentiation and integration (expectation) as it is supported by Leibniz integral rule. We also let $\bar f=\EE f$.


Since $\nabla_{b_L}\bar h=\EE[\nabla\bar\sigma], \nabla_{w_L}\bar h=\EE[\nabla\bar\sigma]u_{L-1}^{\top}, \nabla_{u_{L-1}}\EE_{Z_L}h=w_L^{\top}\EE[\nabla\bar\sigma]$. The smoothness of $\bar\sigma$ implies that $\nabla_{b_L}\bar h, \nabla_{w_L}\bar h, \nabla_{u_{L-1}}\bar h$ are $L_{b_L}^{\bar h}, L_{w_L}^{\bar h}, L_{u_{L-1}}^{\bar h}$-Lipschitz continuous respectively, with
\begin{align*}
    &L_{b_L}^{\bar h}:=L_{\bar{\sigma}}(c),\\
    &L_{w_L}^{\bar h}:=L_{b_L}^{\bar h}\EE[\|u_{L-1}\|],\\
    &L_{u_{L-1}}^{\bar h}:=L_{b_L}^{\bar h}\|w_L\|,\\
    &\|\nabla_{u_{L-1}}\bar h\|\leq L_{\sigma}^0\|w_L\|.
\end{align*}
Since $\sigma$ is differentiable almost everywhere, $\bar\sigma(w_L\sigma(\cdot + Z_{L-1})+b_L)$ is differentiable almost everywhere and thus is smooth in expectation of $Z_{L-1}$. By some calculation, we can show that 
\begin{align*}
    &L_{b_{L-1}}^{\bar h}=L_{u_{L-1}}^{\bar h}(L_{\sigma}^0)^2+L_{\sigma}^0\|w_L\|L_{\bar{\sigma}}(c_l)\\
    &L_{w_{L-1}}^{\bar h}=L_{b_{L-1}}^{\bar h}\EE[\|u_{L-1}\|]\\
    &L_{u_{L-2}}^{\bar h}=L_{b_{L-1}}^{\bar h}\|w_{L-1}\|.
\end{align*}
Following the similar steps, we can obtain that
\begin{align*}
    &L_{b_l}^{\bar h}=L_{u_l}^{\bar h}(L_{\sigma}^0)^2+L_{\sigma}^0\|w_L\|\cdots L_{\sigma}^0\|w_{l+1}\|L_{\bar\sigma}(c_l),\\
    &L_{w_l}^{\bar h}=L_{b_l}^{\bar h}\EE[\|u_{l-1}\|],\\
    &L_{u_{l-1}}^{\bar h}=L_{b_l}^{\bar h}\|w_l\|.
\end{align*}

As long as the overall loss $\ell(\theta_0, h_1,\ldots, h_M;y)$ is smooth w.r.t. $\theta_0, h_1,\ldots, h_M$, we can extend our results to show that it is smooth in the local parameters $\theta_1,\ldots,\theta_M$. 
Taking $u_l$ from $h_m$ as example, that is
\begin{align*}
    L^{\bar\ell}_{\theta_m}=L^{\bar\ell}_{h_m}(L_{h_m}^0)^2+L_{\ell}^0L^{\bar h_m}_{\theta_m}
\end{align*}
we can extend our results to show that $F_c(\theta_0,\vtheta)=\frac{1}{N}\sum\limits_{n=1}^N\EE\ell(\theta_0,h_{n,1},\ldots,h_{n,M};y_n)+\sum\limits_{m=1}^Mr(\theta_m)$ is smooth, where the expectation is taken with respect to all the random neurons in local embedding vectors $h_1,\ldots, h_M$. Specifically, the smoothness of $F_c$ is given by 
\begin{equation}
    L^{F_c}_{\theta_m}=L^{\bar\ell}_{h_m}(L_{\bar h_m}^0)^2+L_{\ell}^0\sum_{l=1}^L(L_{w_l}^{\bar h}+L_{b_l}^{\bar h})+L^{r}_{\theta_m}
\end{equation}
where $L^{r}_{\theta_m}$ is the smoothness constant of the regularizer w.r.t. $\theta_m$; $L_{b_l}^{\bar h}$ and $L_{w_l}^{\bar h}$ are the smoothness constants of the perturbed local embedding $h$ w.r.t. the bias $b_l$ and weight $w_l$.

\subsection{The objective difference after local perturbation}
Now we evaluate the difference between $F_{c}(\theta_0,\vtheta)$ and $F(\theta_0,\vtheta)$. Note that
\begin{align}\label{eqn:smoothdiff}
    |F_{c}(\theta_0,\vtheta)-F(\theta_0,\vtheta)|^2
    &=\left|\frac{1}{N}\sum_{n=1}^N(\EE_{\mathbf{Z}}\ell(\theta_0,h_{n,1}',\ldots,h_{n,M}’;\mathbf {Z})-\ell(\theta_0,h_{n,1},\ldots,h_{n,M}))\right|^2\nonumber\\
    &\leq\frac{1}{N}\sum_{n=1}^N\EE_{\mathbf{Z}}[|\ell(\theta_0,h_{n,1}',\ldots,h_{n,M}';\mathbf{Z})-\ell(\theta_0,h_{n,1},\ldots,h_{n,M})|^2]\nonumber\\
    &\leq\frac{M}{N}\sum_{n=1}^NL_{\ell}^2\EE_{\mathbf{Z}}[\|h_{n,m}'-h_{n,m}\|^2]
\end{align}
where $h_{n,m}$ and $h_{n,m}'$ correspond to the outputs of \eqref{eqn:localembed} and \eqref{eqn:pertlocalembed}, respectively. 
Since we have that
\begin{align*}
    &\|h_{n,m}'-h_{n,m}\|=\|u_L'-u_L\|\leq L_{\sigma_L}(\|w_L\|\|u_{L-1}'-u_{L-1}\|+\|Z_L\|)\leq\cdots\leq\sum\limits_{j=1}^L\left(\prod\limits_{l=j}^{L} L_{\sigma_l}\|w_l\|\right)\|Z_l\|
\end{align*}    
and thus it follows that
\begin{align*}    
    &\|h_{n,m}'-h_{n,m}\|^2\leq\left(\sum\limits_{j=1}^L\prod\limits_{l=j}^LL_{\sigma_l}^2\|w_l\|^2\right)\left(\sum\limits_{j=1}^L\|Z_j\|^2\right).
\end{align*}    
Taking expectation on both side, we have
\begin{align*}    
    &\EE[\|h_{n,m}'-h_{n,m}\|^2]\leq\left(\sum\limits_{j=1}^L\prod\limits_{l=j}^LL_{\sigma_l}^2\|w_l\|^2\right)\left(\sum\limits_{j=1}^{L-1}{c_l^2}+c^2\right).
\end{align*}
Plugging into \eqref{eqn:smoothdiff}, we arrive at
\begin{align*}
    |F_{c}(\theta_0,\vtheta)-F(\theta_0,\vtheta)|\leq M\left(\sum\limits_{j=1}^L\prod\limits_{l=j}^LL_{\sigma_l}^2\|w_l\|^2\right)^{\frac{1}{2}}\left(\sum\limits_{j=1}^{L-1}{c_l^2}+c^2\right)^{\frac{1}{2}}.
\end{align*}

\section{Proof of Theorem \ref{thm:privacy}}
Let $u_l,u_l'$ denote the the outputs of $l$-th layer with inputs $u_0=x,x'$. Under the assumptions that $Z_l\sim\mathcal U[-c_l/2,c_l/2]$,$\sigma_l$ is $L_{\sigma_l}$-Lipschitz continuous for $l=1,\ldots, L-1$, we can derive that
\begin{align*}
    \|w_Lu_{L-1}-w_Lu_{L-1}'\|
    &\leq \|w_L\|L_{\sigma_{L-1}}(\|w_{L-1}\|\|u_{L-2}-u_{L-2}'\|+\sqrt{d_{L-1}}c_{L-1})\\
    &\leq \|w_L\|\prod\limits_{l=1}^{L-1}L_{\sigma_l}\|w_l\|\|x-x'\|+\|w_L\|\sum\limits_{l=1}^{L-1}\left(\prod\limits_{j=1}^{l}L_{\sigma_j}\sqrt{d_j}\right)c_{l}\\
   &:=\bar B.
\end{align*}
Consider the linear operation of $L$-th layer $\mathcal M(u_{L-1})=w_Lu_{L-1}+b_L+Z_L$ which is a random mechanism defined by $Z_L\sim\mathcal N(0,\nu^2)$. Since differential privacy is immune to post-processing \cite{dwork2014algorithmic}, $\sigma_L\circ\mathcal{M}$ does not increase the privacy loss compared with $\mathcal M$. According to Theorem 1 in \cite{abadi2016}, Algorithm \ref{alg:Async-SGD-FL} is $(\varepsilon,\delta)$-differentially private if $\nu=c\frac{q\sqrt{T\log(1/\delta)}}{\varepsilon}$.

\section{Simulation details}
In this section, we present the details of our simulations, and provide the additional test results. 
\subsection{Simulation environment}
We conducted our simulations on a deep learning workstation with 2 \textsf{Nvidia Titan V} and 2 \textsf{Nvidia GeForce RTX 2080 Ti} GPUs. 
Codes are written using \textsf{Python 3.6} and \textsf{Tensorflow 2.0}. 
\subsection{VAFL for federated logistic regression}
\noindent\textbf{Data allocation.}
The datasets we choose are CIFAR-10, Parkinson Disease, MNIST and Fashion MNIST. The batch size is selected to be approximate 0.01 fraction of the entire training dataset.
The data are uniformly distributed among $M=8$ clients for CIFAR-10, $M=3$ for Parkinson Disease, and $M=7$ for both MNIST and Fashion MNIST. 

\noindent\textbf{Stepsize.}
The stepsize is $\eta=1\times10^{-2}$ for Parkinson Disease, $\eta=2\times10^{-4}$ for CIFAR-10, and $\eta=1\times10^{-4}$ for both MNIST and Fashion MNIST.  

\noindent\textbf{Random delay.}
The random delay follows a Poisson distribution with client-specific parameters to reflect heterogeneity. 
The delay on each worker $m$ follows the Poisson distribution with parameter $2m$ and scaled by $1/2M$, where $M$ is the number of workers and $m$ is the worker index.
The expectation of maximum worker delay is one second. 

\noindent\textbf{Perturbation.}
The noise added to the output of each local client follows the Gaussian distribution of each task is $\mathcal N(0,0.01)$ for CIFAR-10, MNIST and Fashion MNIST and $\mathcal N(0,1)$ for Parkinson Disease.

For each task, we run the algorithms sufficiently many epochs and record the training loss. Testing accuracy and wall clock time are recorded at the end of each epoch.

 \begin{figure}[t]
\centering
\includegraphics[width=.45\textwidth]{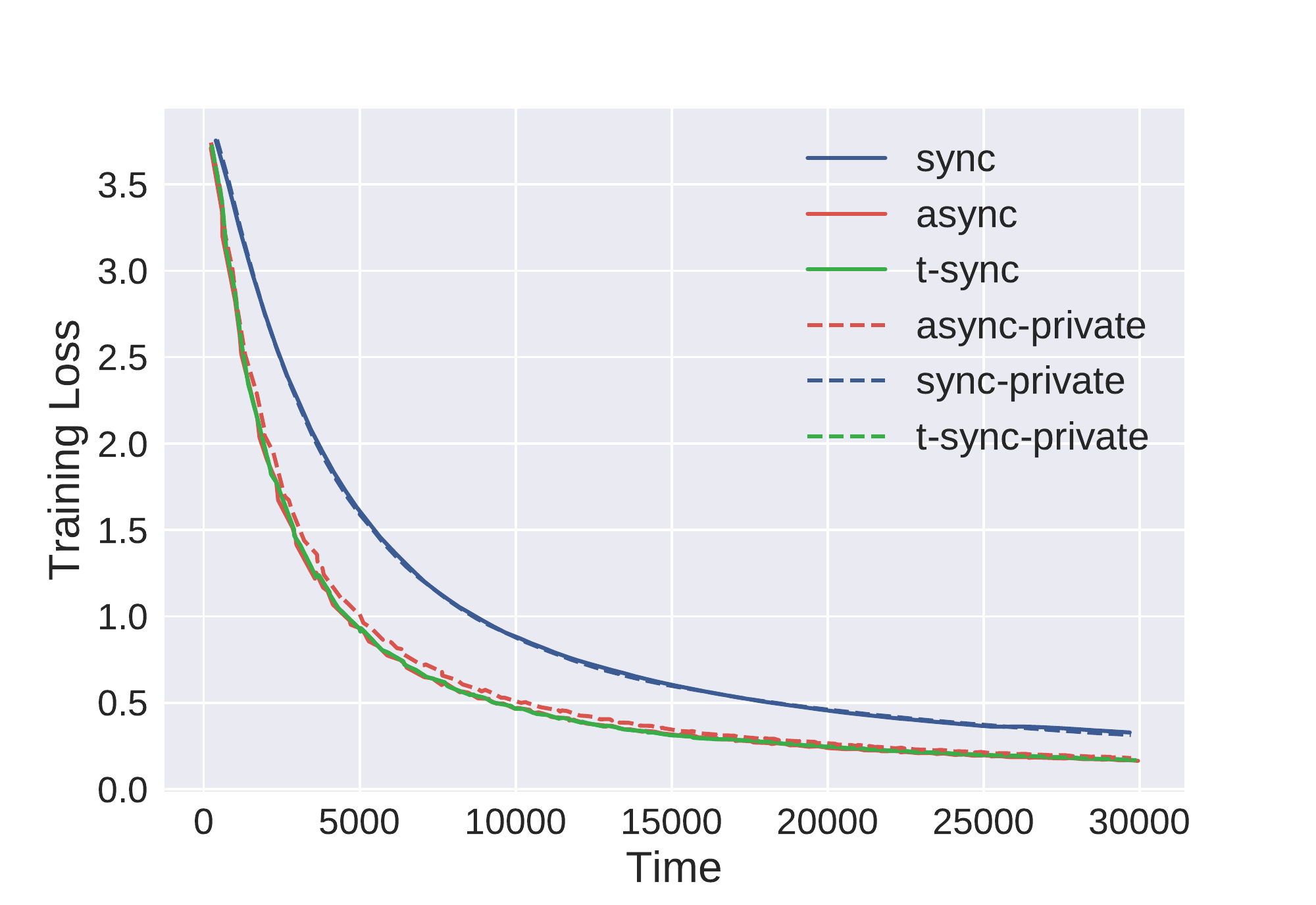}
\vspace{-0.2cm}
\caption{Training loss of VAFL with nonlinear local embedding on \textit{ModelNet40} dataset.}
\label{fig:Nonlinear_loss2}
\vspace{-0.2cm}
\end{figure}

\begin{figure}[t]
\centering
\includegraphics[width=.45\textwidth]{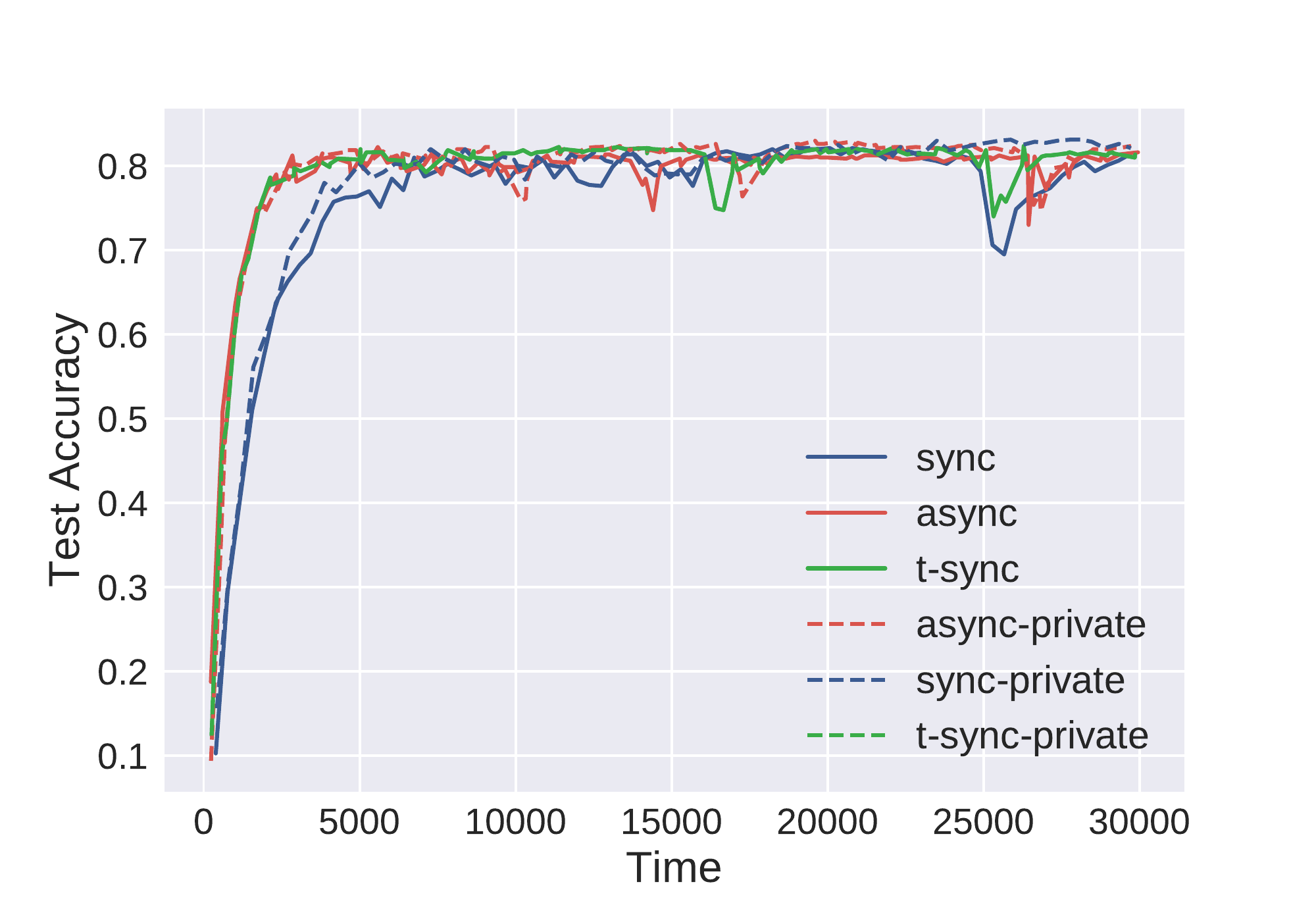}
\vspace{-0.2cm}
\caption{Testing accuracy of VAFL with nonlinear local embedding on \textit{ModelNet40} dataset.}
\label{fig:Nonlinear_acc2}
\vspace{-0.3cm}
\end{figure}

\begin{figure}[t]
\centering
\includegraphics[width=.45\textwidth]{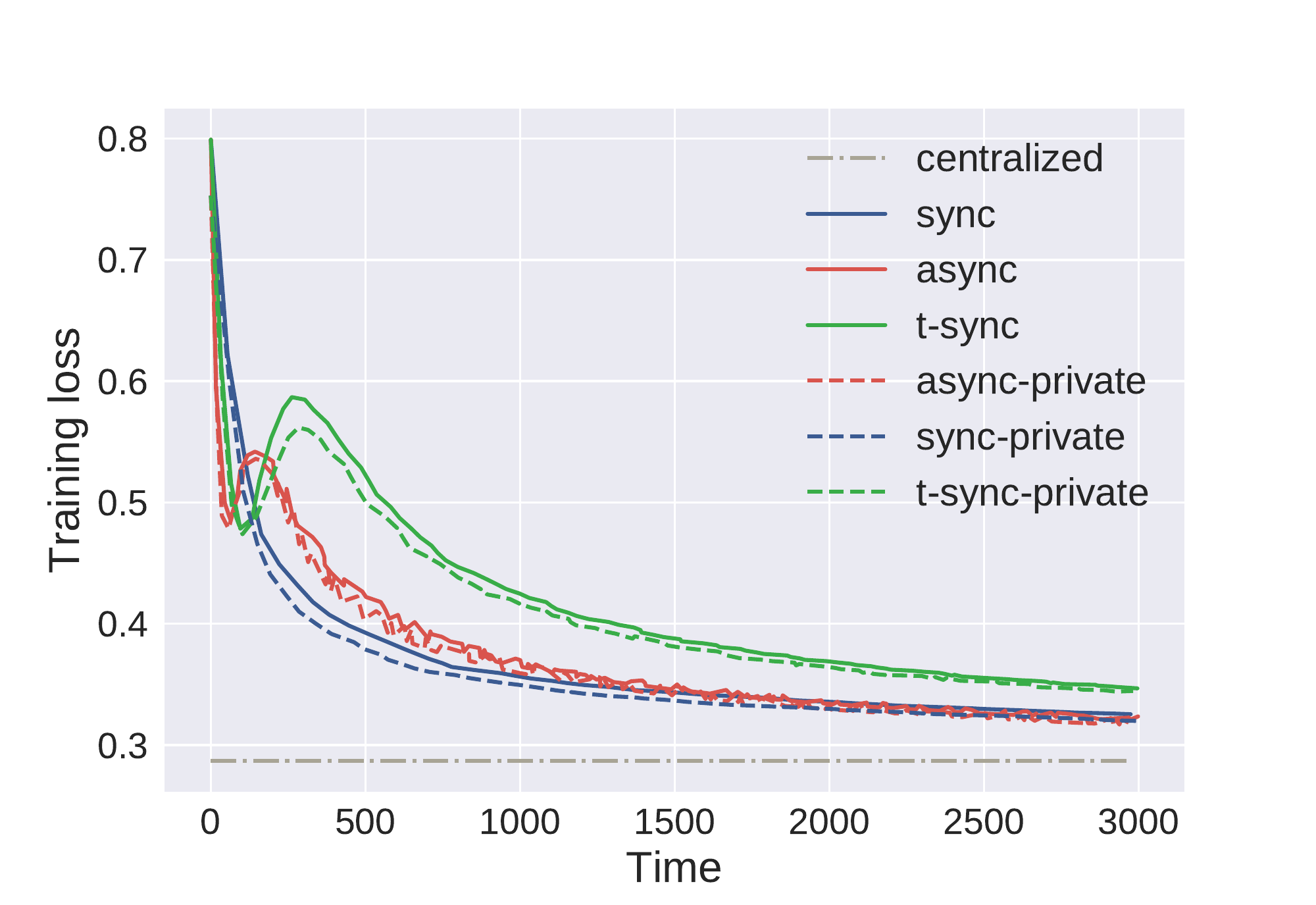}
\vspace{-0.2cm}
\caption{Training loss of VAFL with local LSTM embedding on \textit{MIMIC-III} critical care dataset.}
\label{fig:mortality_loss2}
\vspace{-0.2cm}
\end{figure}

\begin{figure}[t]
\centering
\includegraphics[width=.45\textwidth]{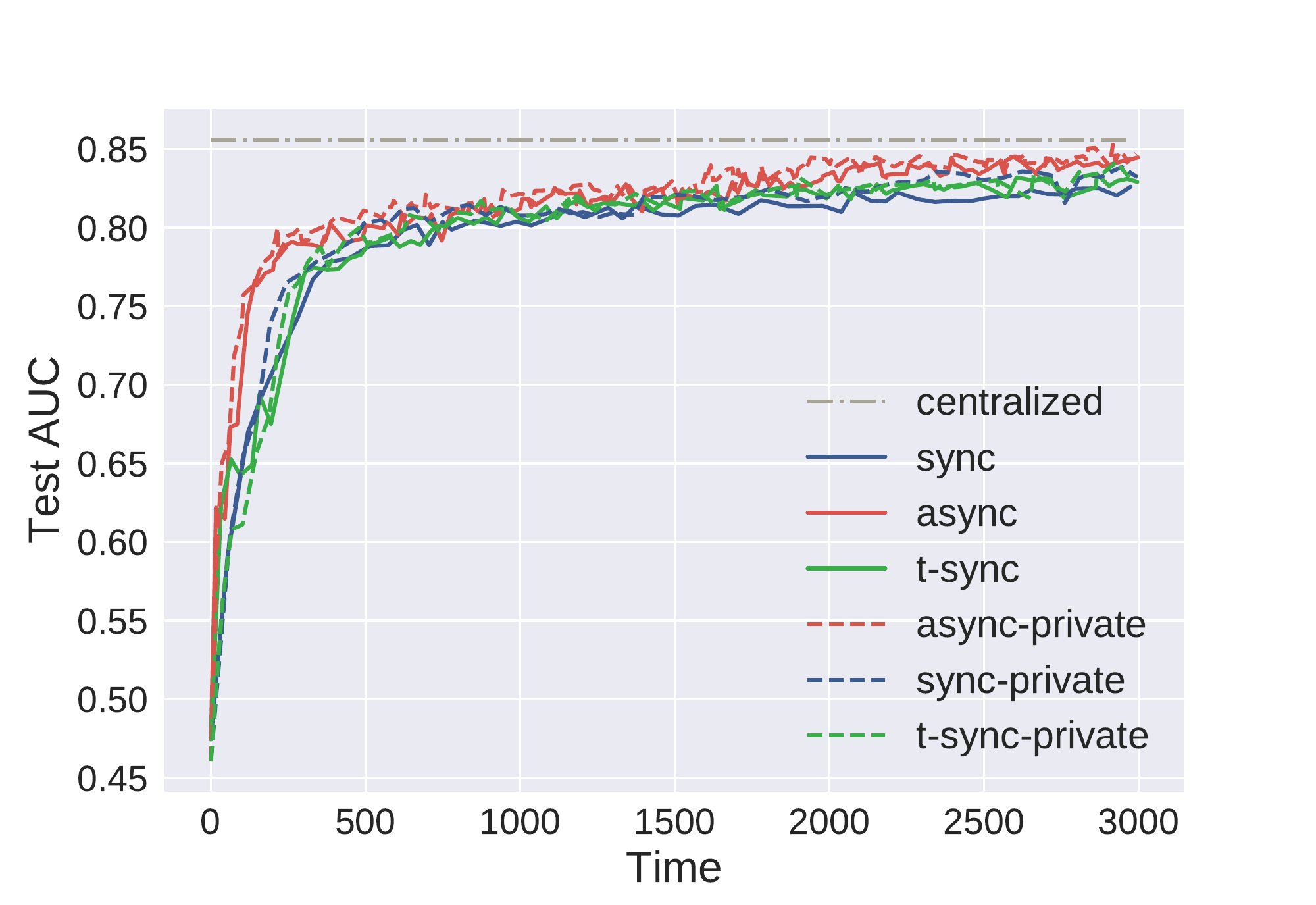}
\vspace{-0.2cm}
\caption{AUC curve of VAFL with local LSTM embedding on \textit{MIMIC-III} clinical care dataset.}
\label{fig:mortality_auc2}
\vspace{-0.3cm}
\end{figure}

\subsection{VAFL for federated deep learning}
\subsubsection{Training on ModelNet40 dataset}

\noindent\textbf{Local embedding structure.}
We train a convolutional neural network-based model consisting of two parts: the local embedding models and the server model. Each local model is a 7-layer convolutional neural network. 
The server part is a centralized 3-layer fully connect neural network. 

\noindent\textbf{Vertical data allocation.}
The data we choose is ModelNet40 and we vertically distributed images of the objects in the dataset from 12 angles and assign to each local client. Each local client deals with the data assigned by their local convolutional network and generate a vector whose dimension is $512$ as the local output. 

\noindent\textbf{Random delay.}
The random delay follows exponential distribution with client-specific parameters to reflect heterogeneity. 
For each worker $m$, the delay follows the exponential distribution with parameter $m$.

\noindent\textbf{Random perturbation.}
We use ReLU as the local embedding activation function. We add a random noise on the output of each local embedding convolutional layer. The noises follow the following distributions: $\mathcal U(-0.1,0.1)$ (the first two layers), $\mathcal U(-0.01,0.01)$ (the other convolutional layers except the last layer) and $\mathcal N(0,1)$ (the last convolutional layers).

\noindent\textbf{Server structure.}
The server then combines the $12$ vectors linearly and pass them into the three-layer fully connected neural network and classify into 40 classes. The number of nodes of each layer is $256$, $100$ and $40$.

\noindent\textbf{Learning rate.}
The stepsize of the local embedding update $\eta_m$ is $10^{-3}$ and the server stepsize $\eta_0 = \frac{\eta_m}{M}$ where $M$ is the number of workers.

\subsubsection{Training on MIMIC-III dataset}
MIMIC is an open dataset comprising deidentified health data associated with ~60,000 intensive care unit admissions \cite{johnson2016mimic}. The data are allocated into 4 workers having different feature dimensions. 

\noindent\textbf{Local embedding structure.}
The local embedding part is a two layer LSTM models and the server part is a fully connected layer. The first layer is a bidirectional LSTM and the number of units is 16. The second layer is a normal LSTM layer and the number of units is also 16.

\noindent\textbf{Random delay.}
The random delay follows an exponential distribution with client-specific parameters to reflect heterogeneity. 
The delay on each worker $m$ follows an exponential distribution with parameter $m$.

\noindent\textbf{Random perturbation.}
A random noise following Gaussian distribution $\mathcal N(0,10^{-4})$ is also added on the output of each local embedding layer.

We have also added simulation results that compare all the algorithms with their private counterparts on both ModelNet40 and MIMIC-III datasets.

\end{document}